%% file: simulated_tempering.tex
\documentclass[11pt]{article}  

\input{packages_only_arxiv.tex}

\input{theorems_std.tex}
\input{macros.tex}

\newcommand{\Anote}[1]{}
\newcommand{\Rnote}[1]{}
\newcommand{\Hnote}[1]{}

\newcommand{\citep}[1]{\cite{#1}}
\newcommand{\weight}[0]{w} 

\begin{document}

\title{Beyond Log-concavity: Provable Guarantees for Sampling Multi-modal Distributions using Simulated Tempering Langevin Monte Carlo}

\author{Rong Ge\thanks{Duke University, Computer Science Department \texttt{rongge@cs.duke.edu}}, Holden Lee\thanks{Princeton University, Mathematics Department \texttt{holdenl@princeton.edu}}, Andrej Risteski\thanks{Massachusetts Institute of Technology, Applied Mathematics and IDSS \texttt{risteski@mit.edu}}}

\date{\today}
\maketitle
\begin{abstract}
In the last several years, 
provable guarantees for iterative optimization algorithms like gradient descent and expectation-maximization in non-convex settings have become a topic of intense research 
in the machine learning community. These works have shed light on the practical success of these algorithms in many unsupervised learning settings such as matrix completion, sparse coding, and learning latent variable Bayesian models.

Another elementary task at inference-time in Bayesian settings, besides model learning, is sampling from distributions that are only specified up to a partition function (i.e., constant of proportionality). As a concrete example, in latent-variable models, sampling the posterior on the latent variables is how a model is \emph{used} after it has been learned. Similar worst-case theoretical issues plague this task as do the learning one: without any assumptions, sampling (even approximately) can be \#P-hard. However, few works have provided ``beyond worst-case'' guarantees for such settings.  


The analogue of ``convexity'' for inference is ``log-concavity'': for log-concave distributions, classical results going back to \cite{bakry1985diffusions} show that natural continuous-time Markov chains called \emph{Langevin diffusions} mix in polynomial time. The most salient feature of log-concavity violated in practice is uni-modality: commonly, the distributions we wish to sample from are multi-modal. In the presence of multiple deep and well-separated modes, Langevin diffusion suffers from torpid mixing.

We address this problem by combining Langevin diffusion with \emph{simulated tempering}. The result is a Markov chain that mixes more rapidly by transitioning between different temperatures of the distribution. 
We analyze this Markov chain for the canonical multi-modal distribution: a mixture of gaussians (of equal variance). The algorithm based on our Markov chain provably samples from distributions that are close to mixtures of gaussians, given access to the gradient of the log-pdf.
For the analysis, we use a spectral decomposition theorem for graphs~ \cite{gharan2014partitioning} and a Markov chain decomposition technique \cite{madras2002markov}.   
\end{abstract}

\newpage

\tableofcontents

\input{overview}
\input{proof_overview}

\input{bounding_temperinggap}

\input{defining_partitions} 
\input{highest_temp}

\input{discretization}

\input{estimates_partition}

\section{Acknowledgements} 

This work was done in part while the authors were visiting the Simons Institute for the Theory of Computing. We thank Matus Telgarsky and Maxim Raginsky for illuminating conversations in the early stages of this work.

\printbibliography
\appendix
\input{mc_background}
\input{example}

\input{tolerance_perturbation}

\input{other_simulated_tempering}

\end{document}

%% file: packages_only_arxiv.tex
\usepackage{etex}

\usepackage{fullpage}

\usepackage{amsmath}
\usepackage{amssymb}
\usepackage{amsthm}
\usepackage{array}
\usepackage{bbm}
\usepackage{cancel}
\usepackage{cmap}
\usepackage{enumerate}
\usepackage{enumitem}
\usepackage{fancyhdr}
\usepackage{mathdots}
\usepackage{mathtools}
\usepackage{mathrsfs}
\usepackage{hyperref}
\usepackage{stackrel}
\usepackage{stmaryrd}
\usepackage{tabularx}
\usepackage{tikz}
\usepackage{ctable}
\usepackage{titletoc}
\usepackage{url}
\usepackage{verbatim}
\usepackage{wasysym}
\usepackage{wrapfig}
\usepackage{yhmath}
\usepackage[all,cmtip]{xy}

\usepackage{float}

\usepackage{hyperref}

\usepackage{algorithm}
\usepackage{algorithmic}
\usepackage[%
	firstinits=true,
	doi=false,
	url=false,
	isbn=false,
	backend=biber,
	style=alphabetic,hyperref]{biblatex}

\usetikzlibrary{calc,trees,positioning,arrows,chains,shapes.geometric,%
    decorations.pathreplacing,decorations.pathmorphing,shapes,%
    matrix,shapes.symbols,shadows,fadings}



%% file: theorems_std.tex
\newtheorem{thm}{Theorem}[section]
\newtheorem*{thm*}{Theorem}

\newtheorem*{prb*}{Problem}

\newtheorem{asm}[thm]{Assumptions}

\newtheorem*{ax*}{Axiom}

\newtheorem*{clm*}{Claim}

\newtheorem*{conj*}{Conjecture}

\newtheorem{df}[thm]{Definition}
\newtheorem*{df*}{Definition}

\newtheorem*{ex*}{Example}

\newtheorem{lem}[thm]{Lemma}
\newtheorem*{lem*}{Lemma}

\newtheorem*{pos*}{Postulate}
\newtheorem{pr}[thm]{Proposition}
\newtheorem*{pr*}{Proposition}

\newtheorem*{qu*}{Question}
\newtheorem{rem}[thm]{Remark}
\newtheorem*{rem*}{Remark}

%% file: macros.tex
\def\shownotes{1}  \ifnum\shownotes=1
\newcommand{\authnote}[2]{$\ll$\textsf{\footnotesize #1 notes: #2}$\gg$}
\else
\newcommand{\authnote}[2]{}
\fi

\newcommand{\st}[0]{\text{st}}

\newcommand{\cE}[0]{\mathcal{E}}

\newcommand{\E}[0]{\mathbb{E}}
\newcommand{\EE}[0]{\mathop{\mathbb E}}

\newcommand{\cL}[0]{\mathcal{L}}
\newcommand{\sL}[0]{\mathscr{L}}
\newcommand{\N}[0]{\mathbb{N}}

\newcommand{\cP}[0]{\mathcal{P}}
\newcommand{\Pj}[0]{\mathbb{P}}

\newcommand{\R}[0]{\mathbb{R}}

\newcommand{\one}[0]{\mathbbm{1}}

\newcommand{\al}[0]{\alpha}
\newcommand{\be}[0]{\beta}
\newcommand{\ga}[0]{\gamma}

\newcommand{\de}[0]{\delta}
\newcommand{\De}[0]{\Delta}
\newcommand{\ep}[0]{\varepsilon}

\newcommand{\la}[0]{\lambda}

\newcommand{\rh}[0]{\rho}

\newcommand{\Om}[0]{\Omega}
\newcommand{\si}[0]{\sigma}

\newcommand{\dellarge}{\Delta} 
\newcommand{\delsmall}{\tau}

\newcommand{\nin}[0]{\not\in}

\newcommand{\sub}[0]{\subset}

\newcommand{\subeq}[0]{\subseteq}

\newcommand{\bs}[0]{\backslash}
\newcommand{\iy}[0]{\infty}

\newcommand{\rc}[1]{\frac{1}{#1}}
\newcommand{\prc}[1]{\pa{\rc{#1}}}

\newcommand{\fc}[2]{\frac{#1}{#2}}
\newcommand{\sfc}[2]{\sqrt{\frac{#1}{#2}}}
\newcommand{\pf}[2]{\pa{\frac{#1}{#2}}}

\newcommand{\ddd}[1]{\frac{d}{d #1}}

\newcommand{\nb}[0]{\nabla}

\newcommand{\dy}{\,dy}
\newcommand{\dx}{\,dx}

\newcommand{\ab}[1]{\left| {#1} \right|}
\newcommand{\an}[1]{\left\langle {#1}\right\rangle}
\newcommand{\ba}[1]{\left[ {#1} \right]}
\newcommand{\bc}[1]{\left\{ {#1} \right\}}

\newcommand{\pa}[1]{\left( {#1} \right)}

\newcommand{\ve}[1]{\left\Vert {#1}\right\Vert}

\newcommand{\ved}[0]{\ve{\cdot}}
\newcommand{\set}[2]{\left\{{#1}:{#2}\right\}}

\newcommand{\ol}[1]{\overline{#1}}

\newcommand{\ub}[2]{\underbrace{#1}_{#2}}

\newcommand{\wt}[1]{\widetilde{#1}}
\newcommand{\wh}[1]{\widehat{#1}}

\newcommand{\amin}{\operatorname{argmin}}

\newcommand{\diam}{\operatorname{diam}}

\newcommand{\Gap}{\operatorname{Gap}}

\newcommand{\KL}[0]{\mbox{KL}}

\newcommand{\poly}{\operatorname{poly}}

\newcommand{\spn}{\operatorname{span}}

\newcommand{\Supp}{\operatorname{Supp}}

\newcommand{\Var}[0]{\operatorname{Var}}

\newcommand{\Vol}[0]{\text{Vol}}

\providecommand{\cal}[1]{\mathcal{#1}}
\renewcommand{\cal}[1]{\mathcal{#1}}

\newcommand{\pull}[9]{
#1\ar@/_/[ddr]_{#2} \ar@{.>}[rd]^{#3} \ar@/^/[rrd]^{#4} & &\\
& #5\ar[r]^{#6}\ar[d]^{#8} &#7\ar[d]^{#9} \\}

\newcommand{\cmp}[9]{
\xymatrix{
#1 \ar[r]^{#4}{#5} \ar@/_2pc/[rr]^{#8}_{#9} & #2 \ar[r]^{#6}_{#7} & #3
}
}

\newcommand{\ha}[1]{\ar@{^(->}[#1]}
\newcommand{\ls}[1]{\ar@{-}[#1]}
\newcommand{\sj}[1]{\ar@{->>}[#1]}
\newcommand{\aq}[1]{\ar@{=}[#1]}
\newcommand{\acir}[1]{\ar@{}[#1]|-{\textstyle{\circlearrowright}}}
\newcommand{\acil}[1]{\ar@{}[#1]|-{\textstyle{\circlearrowleft}}}
\newcommand{\ard}[1]{\ar@{.>}[#1]}
\newcommand{\mt}[1]{\ar@{|->}[#1]}
\newcommand{\inm}[1]{\ar@{}[#1]|-{\in}}
\newcommand{\inr}{\ar@{}[d]|-{\rotatebox[origin=c]{-90}{$\in$}}}
\newcommand{\inl}{\ar@{}[u]|-{\rotatebox[origin=c]{90}{$\in$}}}

\newcommand{\maxr}[2]{\max_{\scriptsize \begin{array}{c}{#1}\\{#2}\end{array}}}
\newcommand{\minr}[2]{\min_{\scriptsize \begin{array}{c}{#1}\\{#2}\end{array}}}
\newcommand{\trow}[2]{\scriptsize \begin{array}{c}{#1}\\{#2}\end{array}}

\newcommand{\sumo}[2]{\sum_{#1=1}^{#2}}

\newcommand{\beq}[1]{\begin{equation}\llabel{#1}}
\newcommand{\eeq}[0]{\end{equation}}
\newcommand{\bal}[0]{\begin{align*}}
\newcommand{\eal}[0]{\end{align*}}
\newcommand{\ban}[0]{\begin{align}}
\newcommand{\ean}[0]{\end{align}}
\newcommand{\ig}[2]{\begin{center}\includegraphics[scale=#2]{#1}\end{center}}

\newcommand{\fixme}[1]{{\color{red}#1}}
\newcommand{\llabel}[1]{\label{#1}\text{\fixme{\tiny#1}}}

\newcommand{\arxiv}[1]{\url{http://www.arxiv.org/abs/#1}}

\newcommand{\vocab}[1]{\textbf{#1}} 

\allowdisplaybreaks[2]

\DeclareFontFamily{U}{wncy}{}
    \DeclareFontShape{U}{wncy}{m}{n}{<->wncyr10}{}
    \DeclareSymbolFont{mcy}{U}{wncy}{m}{n}
    \DeclareMathSymbol{\Sh}{\mathord}{mcy}{"58}

%% file: overview.tex
\section{Introduction}

In recent years, one of the most fruitful directions of research has been providing theoretical guarantees for optimization in non-convex settings. In particular, a routine task in both unsupervised and supervised learning  is to use training data to fit the optimal parameters for a model in some parametric family. Theoretical successes in this context range from analyzing tensor-based approaches using method-of-moments, to iterative techniques like gradient descent, EM, and variational inference 
in a variety of models. These models include topic models \cite{anandkumar2012spectral, arora2012topic, arora2013practical, awasthi2015some}, dictionary learning \cite{arora2015simple, agarwal2014learning}, gaussian mixture models \cite{hsu2013learning}, and Bayesian networks \cite{arora2016provable}. 


Finding maximum likelihood values of unobserved quantities via optimization is reasonable in many learning settings, as when the number of samples is large maximum likelihood will converge to the true values of the quantities. However, for Bayesian inference problems (e.g. given a document, what topics is it about) the number of samples can be limited and maximum likelihood may not be well-behaved \cite{sontag2011complexity}. In these cases we would prefer to {\em sample} from the posterior distribution. \Rnote{added the paragraph to try to discuss why sampling is important.} 
In more generality, the above (typical) scenario is sampling from the \emph{posterior} distribution over the latent variables of a latent variable Bayesian model whose parameters are known. In such models, the observable variables $x$ follow a distribution $p(x)$ which has a simple and succinct form \emph{given} the values of some latent variables $h$, i.e., the joint $p(h,x)$ factorizes as $p(h) p(x|h)$ where both factors are explicit. Hence, the \emph{posterior} distribution $p(h|x) $ has the form $p(h|x) = \frac{p(h) p(x|h)}{p(x)}$. Even though the numerator is easy to evaluate,
without structural assumptions such distributions are often hard to sample from (exactly or approximately). The difficulty is in evaluating the denominator $p(x)=\sum_h p(h)p(x|h)$, which can be NP-hard to do even approximately even for simple models like topic models \cite{sontag2011complexity}.

The sampling analogues of convex functions, which are arguably the widest class of real-valued functions for which optimization is easy, are \emph{log-concave} distributions, i.e. distributions of the form $p(x) \propto e^{-f(x)}$ for a convex function $f(x)$. Recently, there has been renewed interest in analyzing a popular Markov Chain for sampling from such distributions, when given gradient access to $f$---a natural setup for the posterior sampling task described above. In particular, a Markov chain called \emph{Langevin Monte Carlo} (see Section~\ref{sec:overview-l}), popular with Bayesian practitioners, has been proven to work, with various rates depending on the properties of $f$ \cite{dalalyan2016theoretical, durmus2016high,dalalyan2017further}.  

Log-concave distributions are necessarily uni-modal: their density functions have only one local maximum, which must then be a global maximum. This fails to capture many interesting scenarios.
Many simple posterior distributions are neither log-concave nor uni-modal, for instance, the posterior distribution of the means for a mixture of gaussians.
In a more practical direction, complicated posterior distributions associated with deep generative models \cite{rezende2014stochastic} and variational auto-encoders \cite{kingma2013auto} are believed to be multimodal as well.  

The goal of this work is to initiate an  exploration of provable methods for sampling ``beyond log-concavity,'' in parallel to optimization ``beyond convexity''. As worst-case results are prohibited by hardness results, we must again make assumptions on the distributions we will be interested in. 
As a first step, in this paper we consider the prototypical multimodal distribution,  a mixture of gaussians. 

\subsection{Our results}
\label{s:assumptions}

We formalize the problem of interest as follows. We wish to sample from a distribution $p: \mathbb{R}^d \to \mathbb{R}$, such that $p(x) \propto e^{-f(x)}$, and we are allowed to query $\nabla f(x)$ and $f(x)$ at any point $x \in \mathbb{R}^d$. 

To start with, we focus on a problem where $e^{-f(x)}$ is the density function of a mixture of gaussians. That is, given centers $\mu_1,\mu_2,\ldots, \mu_n \in \R^d$, weights $w_1,w_2,\ldots, w_n$ ($\sum_{i=1}^n w_i = 1$), variance $\sigma^2$ (all the gaussians are spherical with same covariance matrix $\sigma^2 I$), the function $f(x)$ is defined as\footnote{Note that the expression inside the $\log$ is essentially the probability density of a mixture of gaussians, except the normalization factor is missing. However, the normalization factor can just introduce a constant shift of $f$ and does not really change $\nabla f$.}
\begin{equation}
f(x) = - \log\left(\sum_{i=1}^n w_i \exp\left(-\frac{\|x - \mu_i\|^2}{2\sigma^2}\right)\right). 
\label{eq:f}
\end{equation}

\Rnote{I'm removing the bound $B$ since we can just talk about $\sigma^2$ right?} \Anote{I guess the point was that we can even get away with a \emph{bound} on $\sigma$ even if we don't know sigma exactly,  but I am fine switching to $\sigma$ only. Up to you.} 
Furthermore, suppose that $D$ is such that $\|\mu_i\| \leq D, \forall i \in [n]$. 
We show that there is an efficient algorithm that can sample from this distribution given just access to $f(x)$ and $\nabla f(x)$.

\begin{thm}[main, informal] Given $f(x)$ as defined in Equation (\ref{eq:f}), there is an algorithm with running time $\poly\pa{w_{\min},D,d,\rc{\ep}, \rc{\sigma^2}}$ that outputs a sample from a distribution within TV-distance $\ep$ of $p(x)$.
\end{thm}
\Rnote{$B$ should appear in the dependency right? It is not written in Theorem 3.1.}
\Hnote{The main theorem is phrased in terms of $\sigma$ instead of $B$. Can we just write everything in terms of $\sigma$ and drop $B$?}

Note that because the algorithm does not have direct access to $\mu_1, \mu_2, \ldots, \mu_n$, even sampling from this mixture of gaussians distribution is very non-trivial. Sampling algorithms that are based on making local steps (such as the ball-walk \cite{lovasz1993random,vempala2005geometric} and Langevin Monte Carlo) cannot move between different components of the gaussian mixture when the gaussians are well-separated (see Figure~\ref{fig:mix} left). In the algorithm we use simulated tempering (see Section~\ref{sec:overview-st}), which is a technique that considers the distribution at different temperatures (see Figure~\ref{fig:mix} right) in order to move between different components.

\begin{figure}
\centering
\includegraphics[height=1in]{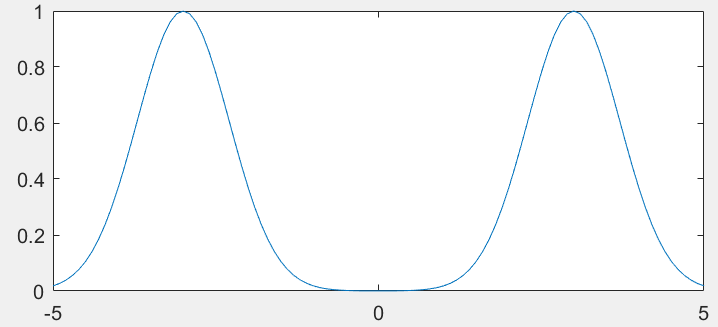}
\includegraphics[height=1in]{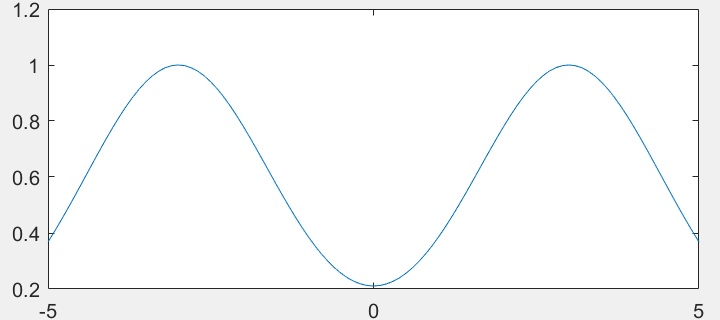}
\caption{Mixture of two gaussians. Left: the two gaussians are well-separated, local sampling algorithm cannot move between modes. Right: Same distribution at high temperature, it is now possible to move between modes.}
\label{fig:mix}
\end{figure}

In Appendix~\ref{sec:examples}, we give a few examples to show some simple heuristics cannot work and the assumption that all gaussians have the same covariance cannot be removed. In particular, we show random initialization is not enough to find all the modes. We also give an example where for a mixture of two gaussians, even if the covariance only differs by a constant multiplicative factor, simulated tempering is known to take exponential time.

Of course, requiring the distribution to be {\em exactly} a mixture of gaussians is a very strong assumption. Our results can be generalized to all functions that are ``close'' to a mixture of gaussians.

More precisely, the function $f$ satisfies the following properties: 
\begin{align}
\exists \tilde{f}:\quad \mathbb{R}^d \to \mathbb{R}
\text{ where } & \ve{\tilde{f} - f}_{\infty} \leq \dellarge \text{   ,   } \ve{\nabla \tilde{f} - \nabla f}_{\infty} \leq \delsmall \text{ and } \nabla^2 \tilde{f}(x) \preceq \nabla^2 f(x) + \delsmall I, \forall x \in \mathbb{R}^d \label{eq:A0}\\
\text{and } \tilde{f}(x) &= -\log\left(\sum_{i=1}^n w_i \exp\left(-\frac{\|x - \mu_i\|^2}{2\sigma^2}\right)\right) \label{eq:tildef}
\end{align}
\Hnote{deleted the $B$}

Intuitively, these conditions show that the density of the  distribution is within a $e^\Delta$ multiplicative factor to an (unknown) mixture of gaussians. Our theorem can be generalized to this case.

\begin{thm}[general case, informal] For function $f(x)$ that satisfies Equations (\ref{eq:A0}) and (\ref{eq:tildef}), there is an algorithm that runs in time $\poly\pa{w_{\min},D,d,\rc{\ep}, \rc{\sigma^2}, e^\De}$ that outputs a sample $x$ from a distribution that has TV-distance at most $\ep$ from $p(x)$.
\label{t:informalperturb}
\end{thm}

\subsection{Prior work}
\label{s:priorwork}

Our algorithm will use two classical techniques in the theory of Markov chains:  \emph{Langevin diffusion}, a chain for sampling from distributions in the form $p(x)\propto e^{-f(x)}$ given only gradient access to $f$ and \emph{simulated tempering}, a heuristic technique used for tackling multimodal distributions. We recall briefly what is known for both of these techniques.  
 
For Langevin dynamics, convergence to the stationary distribution is a classic result \cite{bhattacharya1978criteria}. Understanding the mixing time of the continuous dynamics for log-concave distributions is also a classic result: \cite{bakry1985diffusions, bakry2008simple} show that log-concave distributions satisfy a Poincar\'e and log-Sobolev inequality, which characterize the rate of convergence. Of course, algorithmically, one can only run a ``discretized'' version of the Langevin dynamics, but results on such approaches are much more recent: \cite{dalalyan2016theoretical, durmus2016high,dalalyan2017further} obtained an algorithm for sampling from a log-concave distribution over $\mathbb{R}^d$, and \cite{bubeck2015sampling} gave a algorithm to sample from a log-concave distribution restricted to a convex set by incorporating a projection step. \cite{raginsky2017non} give a nonasymptotic analysis of Langevin dynamics for arbitrary non-log-concave distributions with certain regularity and decay properties. Of course, the mixing time is exponential in general when the spectral gap of the chain is small; furthermore, it has long been known that transitioning between different modes can take an exponentially long time, a phenomenon known as meta-stability \cite{bovier2002metastability, bovier2004metastability, bovier2005metastability}. It is a folklore result that guarantees for mixing extend to distributions $e^{-f(x)}$ where $f(x)$ is a ``nice'' function that is close to a convex function in $L^\iy$ distance; however, this does not address more global deviations from convexity.





It is clear that for distributions that are far from being log-concave and many deep modes, additional techniques will be necessary. Among many proposed heuristics for such situations is simulated tempering, which effectively runs multiple Markov chains, each corresponding to a different temperature of the original chain, and ``mixes'' between these different Markov chains. The intuition is that the Markov chains at higher temperature can move between modes more easily, and if one can ``mix in'' points from these into the lower temperature chains, their mixing time ought to improve as well. Provable results of this heuristic are however few and far between.  
\cite{woodard2009conditions, zheng2003swapping} lower-bound the spectral gap for generic simulated tempering chains. The crucial technique our paper shares with theirs is a Markov chain decomposition technique due to \cite{madras2002markov}. However, for the scenario of Section~\ref{s:assumptions} we are interested in, the spectral gap bound in \cite{woodard2009conditions} is exponentially small as a function of the number of modes. Our result will remedy this. 



\section{Preliminaries}

In this section we first introduce notations for Markov chains. More details are deferred to Appendix~\ref{a:markovchain}. Then we briefly discuss Langevin Monte Carlo and Simulated Tempering.

\subsection{Markov chains}

In this paper, we use both discrete time and continuous time Markov chains. In this section we briefly give definitions and notations for discrete time Markov chains. Continuous time Markov chains follow the same intuition, but we defer the formal definitions to Appendix~\ref{a:markovchain}. 

\begin{df}
A (discrete time) Markov chain is $M=(\Om,P)$, where $\Om$ is a measure space and $P(x,y)\dy$ is a probability measure for each $x$.
It defines a random process $(X_t)_{t\in \N_0}$ as follows. If $X_s=x$, then 
\begin{align}
\Pj(X_{s+1}\in A) = P(x,A) :&=\int_A p(x,y)\dy. 
\end{align}

A \vocab{stationary distribution} is $p(x)$ such that if $X_0\sim p$, then $X_t\sim p$ for all $t$; equivalently, $\int_\Om p(x) P(x,y) \dx = p(y)$. 

A chain is \vocab{reversible} if $p(x)P(x,y) = p(y) P(y,x)$. 
\end{df}

If a Markov chain has a finite number of states, then it can be represented as a weighted graph where the transition probabilities are proportional to the weights on the edges. A reversible Markov chain can be represented as a undirected graph. 

\paragraph{Variance, Dirichlet form and Spectral Gap}
An important quantity of the Markov chain is the {\em spectral gap}.

\begin{df}
For a discrete-time Markov chain $M=(\Om, P)$, let $P$ operate on functions as
\begin{align}
(Pg)(x) = \E_{y\sim P(x,\cdot)} g(y) = \int_{\Om} g(x)P(x,y)\dy.
\end{align}

Suppose $M=(\Om, P)$ has unique stationary distribution $p$.
Let
$\an{g,h}_p :=\int_{\Om} g(x)h(x)p(x)\dx$ and define the Dirichlet form and variance by
\begin{align}
\cal E_M(g,h) &= \an{g, (I-P)h}_p \\
\Var_p(g) &= \ve{g-\int_{\Om} gp\dx}_p^2
\end{align}
Write $\cal E_M(g)$ for $\cal E_M(g,g)$. 
Define the eigenvalues of $M$, $0=\la_1\le \la_2\le \cdots$ to be the eigenvalues of $I-P$ with respect to the norm $\ved_{p}$. 

Define the spectral gap by
\begin{align}
\Gap(M) &= \inf_{g\in L^2(p)} \fc{\cal E_M(g)}{\Var_p(g)}.
\end{align}
\end{df}

In the case of a finite, undirected graph, the function just corresponds to a vector $x\perp \vec{1}$. The Dirichlet form corresponds to $x^\top \cL x$ where $\cL$ is the normalized Laplacian matrix, and the variance is just the squared norm $\|x\|^2$.

The spectral gap controls mixing for the Markov chain. Define the $\chi^2$ distance between $p,q$ by
\begin{align}
\chi_2(p||q) &= \int_\Om \pf{q(x)-p(x)}{p(x)}^2p(x)\dx
= \int_\Om \pf{q(x)^2}{p(x)} - 1.
\end{align}
Let $p^0$ be any initial distribution and $p^t$ be the distribution after running the Markov chain for $t$ steps. Then
\begin{align}\label{eq:gap-mix}
\chi_2(p||p^t) \le (1-G')^t \chi(p||p^0)
\end{align}•
where $G'=\min(\la_2, 2-\la_{\max})$. 

\paragraph{Restrictions and Projections}
Later we will also work with continuous time Markov chains (such as Langevin dynamics, see Section~\ref{sec:overview-l}). In the proof we will also need to consider {\em restrictions} and {\em projections} of Markov chains. Intuitively, restricting a Markov chain $M$ to a subset of states $A$ (which we denote by $M|A$) removes all the states out of $A$, and replaces transitions to $A$ with self-loops. Projecting a Markov chain $M$ to partition $\cP$ (which we denote by $\bar{M}^\cP$) ``merges'' all parts of the partition into individual states. For formal definitions see Appendix~\ref{a:markovchain}.

\paragraph{Conductance and clustering} Finally we define conductance and clusters for Markov chains. These are the same as the familiar concepts as in undirected graphs.

\begin{df}\label{df:conduct}
Let $M=(\Om, P)$ be a Markov chain with unique stationary distribution $p$. Let
\begin{align}
Q(x,y) &= p(x) P(x,y)\\
Q(A,B) & = \iint_{A\times B} Q(x,y)\dx\dy.
\end{align}
(I.e., $x$ is drawn from the stationary distribution and $y$ is the next state in the Markov chain.)
Define the \vocab{(external) conductance} of $S$, $\phi_M(S)$, and the \vocab{Cheeger constant} of $M$, $\Phi(M)$, by
\begin{align}
\phi_M(S) & = \fc{Q(S,S^c)}{p(S)}\\
\Phi(M) &= \min_{S\sub \Om, p(S)\le \rc 2}
\phi_M(S).
\end{align}
\end{df}

The clustering of a Markov chain is analogous to a partition of vertices for undirected graphs. For a good clustering, we require the inner-conductance to be large and the outer-conductance to be small.

\begin{df}\label{df:in-out}
Let $M=(\Om,P)$ be a Markov chain on a finite state space $\Om$. 
We say that $k$ disjoint subsets $A_1,\ldots, A_k$ of $\Om$ are a $(\phi_{\text{in}}, \phi_{\text{out}})$-clustering if for all $1\le i\le k$,
\begin{align}
\Phi(M|_{A_i}) &\ge \phi_{\text{in}}\\
\phi_M(A_i)&\le \phi_{\text{out}}.
\end{align}•
\end{df}

\subsection{Overview of Langevin dynamics} 

\label{sec:overview-l}

Langevin diffusion is a stochastic process, described by the stochastic differential equation (henceforth SDE)
\begin{equation}
dX_t = -\nb f (X_t) \,dt + \sqrt{2}\,dW_t \label{eq:langevinsde}
\end{equation}
where $W_t$ is the Wiener process. 
The crucial (folklore) fact about Langevin dynamics is that Langevin dynamics converges to the stationary distribution given by $p(x) \propto e^{-f(x)}$.  
Substituting $\be f$ for $f$ in~\eqref{eq:langevinsde} gives the Langevin diffusion process for inverse temperature $\be$, which has stationary distribution $\propto e^{-\be f(x)}$. Equivalently it is also possible to consider the temperature as changing the magnitude of the noise:
$$
dX_t = -\nabla f(X_t)dt + \sqrt{2\beta^{-1}}dW_t.
$$

Of course algorithmically we cannot run a continuous-time process, so we run a \emph{discretized} version of the above process: namely, we run a Markov chain where the random variable at time $t$ is described as 
\begin{equation} 
X_{t+1} = X_t - \eta \nb f(X_t)  + \sqrt{2 \eta }\xi_k, \quad \xi_k \sim N(0,I) \label{eq:langevind} 
\end{equation}
where $\eta$ is the step size. (The reason for the $\sqrt \eta$ scaling is that running Brownian motion for $\eta$ of the time scales the variance by $\sqrt{\eta}$.)

The works \cite{dalalyan2016theoretical, durmus2016high, dalalyan2017further} have analyzed the convergence properties (both bias from the stationary distribution, and the convergence rate) for log-concave distributions, while \cite{raginsky2017non} give convergence rates for non-log-concave distributions. Of course, in the latter case, the rates depend on the spectral gap, which is often exponential in the dimension. 


\subsection{Overview of simulated tempering}
\label{sec:overview-st}


Simulated tempering is a technique that converts a Markov chain to a new Markov chain whose state space is a product of the original state space and a temperature. The new Markov chain allows the original chain to change ``temperature'' while maintaining the correct marginal distributions. Given a discrete time Markov chain, we will consider it in $L$ temperatures. Let $[L]$ denote the set $\{1,2,...,L\}$, we define the simulated tempering chain as follows:

\begin{df}
Let $M_i, i\in [L]$ be a sequence of Markov chains with state space $\Om$ and unique stationary distributions $p_i$. 
Let $r_1,\ldots, r_{L}$ be such that 
$$
r_i> 0,\quad \sumo i{L} r_i = 1.
$$

Define the \vocab{simulated tempering Markov chain} with \emph{relative probabilities} $r_i$ as follows. 
The states of $M_{\st}$ are $\Om\times [L]$. 
Suppose the current state is $(x, k)$. 
\begin{enumerate}
\item
With probability $\rc2$, keep $k$ fixed, and update $x$ according to $M_{k}$. We will call this a Type 1 transition. 
\item
With probability $\rc2$, do the following Metropolis-Hastings step: draw $k'$ randomly from $\{0,\ldots, L-1\}$. Then transition to $(x,k')$ with probability
$$
\min \bc{\fc{r_{k'}p_{k'}(x)}{r_kp_k(x)}, 1}
$$
and stay at $(x,k)$ otherwise. We will call this a Type 2 transition. 
\end{enumerate}
\label{df:temperingchain}
\end{df}

\begin{rem}
For the type two transitions, we can instead just pick $k'$ from $\{k-1,k,k+1\}$. This will slightly improve our bounds on mixing time, because the ratio $\fc{p_{k'}(x)}{p_k(x)}$ for $k'\in \{k-1,k+1\}$ is bounded, and can be exponential otherwise. For simplicity, we stick with the traditional definition of the simulated tempering Markov chain.
\end{rem}

The typical setting is as follows. The Markov chains come from a smooth family of Markov chains with parameter $\be\ge 0$, and $M_i$ is the Markov chain with parameter $\be_i$, where $0\le \be_1\le \cdots \be_{L}=1$. (Using terminology from statistical physics, $\be=\rc\tau$ is the inverse temperature.) 
We are interested in sampling from the distribution when $\be$ is large ($\tau$ is small). However, the chain suffers from torpid mixing in this case, because the distribution is more peaked. The simulated tempering chain uses smaller $\be$ (larger $\tau$) to help with mixing.
For us, the stationary distribution at inverse temperature $\be$ is $\propto e^{-\be f(x)}$. 

Of course, the Langevin dynamics introduced in previous section is a continuous time Markov chain. In the algorithm we change it to a discrete time Markov chain by fixing a step size. Another difficulty in running the simulated tempering chain directly is that we don't have access to $p_k$ (because we do not know the partition function). We make use of the flexibility in $r_i$'s to fix this issue. For more details see Section~\ref{sec:overview-alg}.


The crucial fact to note is that the stationary distribution is a ``mixture'' of the distributions corresponding to the different temperatures. Namely:  

\begin{pr} [folklore]
If the $M_{k}$ are reversible Markov chains with stationary distributions $p_k$, then the simulated tempering chain $M$
is a reversible Markov chain with stationary distribution
$$
p(x,i) = r_ip_i(x).
$$
\end{pr}

\section{Our Algorithm} 




\label{sec:overview-alg}

Our algorithm will run a simulated tempering chain, with a polynomial number of temperatures, while running discretized Langevin dynamics at the various temperatures. The full algorithm is specified in Algorithm~\ref{a:mainalgo}. 

As we mentioned before, an obstacle in running the simulated tempering chain is that we do not have access to the partition function. We solve this problem by estimating the partition function from high temperature to low temperature, adding one temperature at a time (see Algorithm~\ref{a:mainalgo}). Note that if the simulated tempering chain mixes and produce good samples, by standard reductions it is easy to estimate the (ratios of) partition functions.

\begin{algorithm}
\begin{algorithmic}
\STATE INPUT: Temperatures $\be_1,\ldots, \be_\ell$; partition function estimates $\wh Z_1,\ldots, \wh Z_\ell$; step size $\eta$, time interval $T$, number of steps $t$.
\STATE OUTPUT: A random sample $x\in \R^d$ (approximately from the distribution $p_\ell(x)\propto e^{\be_\ell f(x)}$).
\STATE Let $(x,k)=(x_0,1)$ where $x_0\sim N(0, \rc{\be}I)$.
\FOR{$s=0\to t-1$}
\STATE (1) With probability $\rc 2$, keep $k$ fixed. Update $x$ according to $x \mapsfrom x - \eta \be_k\nb f(x) +\sqrt{2\eta}\xi_k$, $\xi_k\sim N(0,I)$. Repeat this $\fc{T}{\eta}$ times.
\STATE (2) With probability $\rc2$, make a type 2 transition, where the acceptance ratio is 
$\min \bc{
\fc{e^{-\be_{k'}f(x)}/\wh Z_{k'}}
{e^{-\be_{k}f(x)}/\wh Z_{k}}, 1
}
$.
\ENDFOR
\STATE If the final state is $(x,l)$ for some $x\in \R^d$, return $x$. Otherwise, re-run the chain.
\end{algorithmic}
 \caption{Simulated tempering Langevin Monte Carlo}
 \label{a:stlmc}
\end{algorithm}

\begin{algorithm}
\begin{algorithmic}
\STATE INPUT: A function $ f: \mathbb{R}^d$, satisfying assumption~\eqref{eq:A0}, to which we have gradient access.  
\STATE OUTPUT: A random sample $x \in \mathbb{R}^d$. 
\STATE Let $0\le \be_1<\cdots < \be_L=1$ be a sequence of inverse temperatures satisfying~\eqref{eq:beta1} and~\eqref{eq:beta-diff}. 
\STATE Let $\wh Z_1=1$.
\FOR{$\ell = 1 \to L$}  
 \STATE Run the simulated tempering chain in Algorithm~\ref{a:stlmc} with temperatures 
$\be_1,\ldots, \be_{\ell}$, estimates $\wh Z_1,\ldots, \wh Z_{i}$, step size $\eta$, time interval $T$, and 
number of steps $t$ given by Lemma~\ref{lem:a1-correct}. 
 \STATE If $\ell=L$, return the sample.
 \STATE If $\ell<L$, repeat to get $m=O(L^2\ln \prc{\de})$ samples, and let $\wh{Z_{\ell+1}} = \wh{Z_\ell} \pa{
 \rc m \sumo jm e^{(-\be_{\ell+1}+\be_\ell)f(x_j)}}$.
\ENDFOR
\end{algorithmic}
 \caption{Main algorithm}
\label{a:mainalgo}
\end{algorithm}

Our main theorem is the following. 
\begin{thm}[Main theorem]\label{thm:main}
Suppose $f(x) = -\ln \pa{\sumo in w_i \exp\pa{-\fc{\ve{x-\mu_i}^2}{2\si^2}}}$ on $\R^d$ where $\sumo in w_i=1$, $w_{\min}=\min_{1\le i\le n}w_i>0$, and $D=\max_{1\le i\le n}\ve{\mu_i}$. Then 
Algorithm~\ref{a:mainalgo} with parameters given by Lemma~\ref{lem:a1-correct} produces a sample from a distribution $p'$ with $\ve{p-p'}_1\le \ep$ in time $\poly\pa{w_{\min}, D, d, \rc{\si}, \frac{1}{\ep}}$.
\end{thm}
For simplicity, we stated the theorem for distributions which are exactly mixtures of gaussians. The theorem is robust to $L^\iy$ perturbations as in~\eqref{eq:A0}, we give the more general theorem in Appendix~\ref{sec:perturb}.


%% file: proof_overview.tex
\section{Overview of proof}

We will first briefly sketch the entire proof, and in the subsequent sections expand on all the individual parts. 

The key part of our proof is a new technique for bounding the spectral gap for simulated tempering chain using decompositions (Section~\ref{sec:decomposition}): for each temperature 
we make a partition 
into ``large'' pieces that are well-connected from the inside. If this partition can be done at every temperature, the difference in temperature is small enough, and the chain mixes at the highest temperature, we show the simulated tempering chain also mixes quickly. This is a general theorem for the mixing of simulated tempering chains that may be useful in other settings. 

\begin{figure}[h!]
\centering
\ig{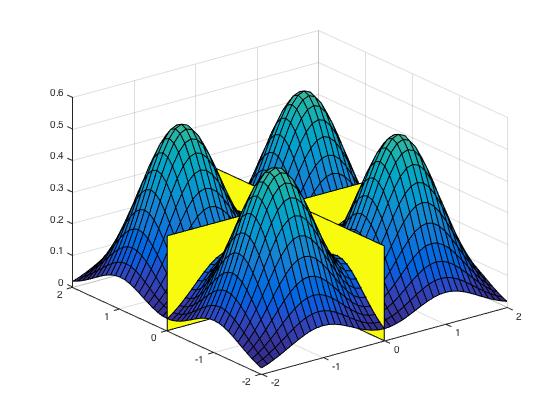}{0.4}
\caption{For a mixture of gaussians, we can partition space into regions where Langevin Monte Carlo mixes well.}
\label{fig:partition}
\end{figure}

We then show that if $f(x)$ is a mixture of gaussians, then indeed the partition exists (Section~\ref{sec:partition}). Here we use spectral clustering techniques developed by \cite{gharan2014partitioning} for finite graphs. The main technical difficulty is in transferring between the discrete and continuous cases.

Finally we complete the proof by showing that
\begin{enumerate}
\item
the Markov chain mixes at the highest temperature (Section~\ref{sec:mixht});
\item
the discretized Markov chain approximates the continuous time Markov chain (Section~\ref{sec:discretizeshort});
\item
the partition functions are estimated correctly which allows us to run the simulated tempering chain (Section~\ref{sec:partitionfunc}).
\end{enumerate}•

At last, in Appendix~\ref{sec:perturb} we prove the arguments are tolerant to $L^{\infty}$ perturbations, i.e. that the algorithm works for distributions that are not \emph{exactly} mixtures of gaussians.   


\subsection{Decomposing the simulated tempering chain} 
\label{sec:decomposition}
First we show that if
there exists a partition $\cal P_i$ for each temperature 
such that
\begin{enumerate}
\item
the Markov chain mixes rapidly within each set of the partition (i.e., $\Gap(M_i|_A)$ is large), and
\item
the sets in the partition are not too small,
\end{enumerate}
and the chain mixes at the highest temperature, 
then the simulated tempering chain mixes rapidly. 


\begin{thm*}[Theorem~\ref{t:temperingnochain}]
Let $M_i=(\Om, P_i), i\in [L]$ be a sequence of Markov chains with state space $\Om$ and stationary distributions $p_i$.
Consider the simulated tempering chain $M=(\Om\times[L], P_{\st})$ %
 with probabilities $(r_i)_{i=1}^{L}$. Let $r = \fc{\min(r_i)}{\max(r_i)}$.  

Let $\cal P_i$
be a partition 
of the ground set $\Om$, for each $i\in [L]$, with $\cal P_1=\{\Om\}$. 

Define the \vocab{overlap parameter} of $(\cal P_i)_{i=1}^L$ to be 
$$
\de((\cal P_i)_{i=1}^L) = \min_{1< i\le L, A\in \cal P} \ba{
\int_{A} \min\{p_{i-1}(x),p_i(x)\} \dx 
}/p_i(A).
$$

Define
$$
p_{\min}=\min_{i,A\in \cal P_i} p_i(A). 
$$

Then the spectral gap of the tempering chain  satisfies
\begin{align}
\Gap(M_{\st}) &\ge\fc{r^4\de^2p_{\min}^2}{32L^4}
\min_{1\le i\le L, A\in \cal P_i} (\Gap(M_{i}|_A)).
\end{align} 
\end{thm*} 
To prove this, we use techniques similar to existing work on  simulated tempering. 
More precisely, similar to the proof in \cite{woodard2009conditions}, we will apply a 
``decomposition'' theorem (Theorem~\ref{thm:gap-product}) for analyzing the mixing time of the simulated tempering chain.

Note that as we are only using this theorem in the analysis, we only need an existential, not a algorithmic result. 
In order to apply this theorem, we will show that there exist good partitions $\mathcal{P}_i$, such that the spectral gap $\Gap(M_{i}|_A)$ within each set is large, and each set in the partition has size $\poly(w_{\min})$. 

\begin{rem}
In Appendix~\ref{app:other}, Theorem~\ref{thm:sim-temp}, we also give a different (incomparable) criterion for lower-bounding the spectral gap that improves the bound in~\cite{woodard2009conditions}, in some cases by an exponential factor. Theorem~\ref{thm:sim-temp} requires that the partitions $\cal P_i$ be successive refinements, but has the advantage of depending on a parameter $\ga((\cal P_i)_{i=1}^L)$ that is larger than $p_{\min}$, and unlike $p_{\min}$, can even be polynomial when the $\cal P_i$ have exponentially many pieces. Theorem~\ref{thm:sim-temp} will not be necessary for the proof of our main theorem.
\end{rem}

\subsection{Existence of partitions} 
\label{sec:partition}

We will show the existence of good partitions $\cal P_i$ for $i\ge 2$ using a theorem of \cite{gharan2014partitioning}. The theorem shows if the $k$-th singular value is large, then it is possible to have a clustering with at most $k$ parts which has high ``inside'' conductance within the clusters and low ``outside'' conductance between the clusters (Definition~\ref{df:in-out}).

\begin{thm*}[Spectrally partitioning graphs, Theorem~\ref{thm:gt14}]
Let $M=(\Om, P)$ be a reversible Markov chain with $|\Om|=n$ states. Let $0=\la_1\le \la_2\le \cdots\le \la_n$ be the eigenvalues of the Markov chain. 

For any $k\ge 2$, if $\la_k>0$, then there exists $1\le \ell\le k-1$ and a $\ell$-partitioning of $\Om$ into sets $P_1,\ldots, P_\ell$ that is a 
$
(\Om(\la_k/k^2), O(\ell^3\sqrt{\la_\ell}))
$-clustering.
\end{thm*}

For a mixture of $n$ gaussians, using  the Poincar\'e inequality for gaussians we can  show that $\lambda_{n+1}$ for the continuous Langevin chain is bounded away from 0. Therefore one would hope to use Theorem~\ref{thm:gt14} above to obtain a clustering. However, there are some difficulties, especially that Theorem~\ref{thm:gt14} only holds for a discrete time, discrete space Markov chain.


To solve the discrete time problem, we fix a time $T$, and consider the discrete time chain where each step is running Langevin for time $T$. To solve the discrete space problem, we note that we can apply the theorem to the Markov chain $M_i$ projected to any partition (See Definition~\ref{df:assoc-mc} for the definition of a projected Markov chain). 
A series of technical lemmas will show that the eigenvalues and conductances do not change too much if we pass to the discrete time (and space) chain.

Another issue is that although the theorem guarantees good inner-conductance, it does not immediately give a lowerbound for the size of the clusters. Here we again use Poincar\'e inequality to show {\em any} small set must have a large outer-conductance, therefore the clustering guaranteed by the theorem cannot have small clusters.
Thus the assumptions of Theorem~\ref{thm:gt14} are satisfied, and we get a partition with large internal conductance and small external conductance for the projected chain (Lemma~\ref{lem:any-partition}). 

By letting the size of the cells in the partition go to 0, we show that the gap of the projected chain approaches the gap of the continuous chain (Lemma~\ref{lem:limit-chain}). Because this lemma only works for compact sets, we also need to show that restricting to a large ball doesn't change the eigenvalues too much (Lemma~\ref{lem:rest-large}).

\subsection{Mixing at highest temperature}
\label{sec:mixht}
Next, we need to show mixing at the highest temperature. Over bounded domains, we could set the highest temperature to be infinite, which would correspond to uniform sampling over the domain. Since we are working over an unbounded domain, we instead compare it to Langevin dynamics for a strictly convex function which is close in $L^\iy$ distance (Lemma~\ref{lem:hitemp}). We use the fact that 
Langevin dynamics mixes rapidly for strictly convex functions, and an $L^{\iy}$ perturbation of $\ep$ affects the spectral gap by at most a factor of $e^{\ep}$. 



%

\subsection{Discretizing the Langevin diffusion}
\label{sec:discretizeshort}
Up to this point, though we subdivided time into discrete intervals of size $T$, in each time interval we ran the \emph{continuous} Langevin chain for time $T$. However, algorithmically we can only run a discretization of Langevin diffusion -- so we need to bound the drift of the discretization from the continuous chain. 

For this, we follow the usual pattern of discretization arguments: if we run the continuous chain for some time $T$ with step size $\eta$, the drift of the discretized chain to the continuous will be $\eta T$. If we have a bound on $T$, this provides a bound for the drift. More precisely, we show: 

\begin{lem*}[Lemma~\ref{l:maindiscretize}] Let $p^t, q^t: \mathbb{R}^d \times [L]  \to \mathbb{R}$ be the distributions after running the simulated tempering chain for $t$ steps, where in $p^t$, for any temperature $i \in L$, the Type 1 transitions are taken according to the (discrete time) Markov kernel $P_T$: running Langevin diffusion for time $T$; in $q^t$, the Type 1 transitions are taken according to running $\frac{T}{\eta}$ steps of the discretized Langevin diffusion, using $\eta$ as the discretization granularity, s.t. $\eta \leq \frac{\sigma^2}{2}$.  
Then, 
\begin{align*} \mbox{KL} (p^t || q^t) \lesssim \frac{\eta^2}{\sigma^6} (D+d) T t^2 + \frac{\eta^2}{\sigma^6} \max_i \E_{x \sim p^0( \cdot, i)}\|x - x^*\|_2^2 + \frac{\eta}{\sigma^4} d t T \eta  \end{align*}
\end{lem*} 

To prove this, consider the two types of steps separately. The Type 2 steps in the tempering chains do not increase the KL divergence between the continuous and discretized version of the chains, and during the Type 1 steps, the increase in KL divergence on a per-step basis can be bounded by using existing machinery on discretizing Langevin diffusion (see e.g. \cite{dalalyan2016theoretical}) along with a decomposition theorem for the KL divergence of a mixture of distributions (Lemma~\ref{l:decomposingKL}).  

We make a brief remark about $\max_i \E_{x \sim p^0( \cdot, i)}\|x - x^*\|_2^2 $: since the means $\mu_i$ satisfy $\|\mu_i\| \leq D$, it's easy to characterize the location of $x^*$ and conclude that bounding this quantity essentially requires that most of the mass in the initial distributions should be concentrated on a ball of size $O(D)$. Namely, the following holds:  

\begin{lem} Let $x^* = \mbox{argmin}_{x \in \mathbb{R}^d} \tilde{f}(x)$. Then, $\|x^*\| \leq \sqrt{2} D$.  
\end{lem} 

\subsection{Estimating the partition functions}
\label{sec:partitionfunc}
Finally, the Metropolis-Hastings filter in the Type 2 step of the simulated tempering chain requires us to estimate the partition functions at each of the temperatures. It is sufficient to estimate the partition functions to within a constant factor, because the gap of the tempering chain depends on the ratio of the maximum-to-minimum probability of a given temperature.  

For this, we run the chain for temperatures $\be_1,\ldots, \be_\ell$, obtain good samples for $p_\ell$, and use them to estimate $Z_{\ell+1}$. We use Lemma~\ref{l:partitionfunc} to show that with high probability, this is a good estimate.


%% file: bounding_temperinggap.tex
\section{Spectral gap of simulated tempering}
\label{s:main}

In this section we prove a lower bound for the spectral gap of simulated tempering given a partition.

\begin{asm}\label{asm}
Let $M_i=(\Om, P_i), i\in [L]$ be a sequence of Markov chains with state space $\Om$ and stationary distributions $p_i$.
Consider the simulated tempering chain $M_{\st}=(\Om\times[L], P_{\st})$ %
 with probabilities $(r_i)_{i=1}^{L}$. Let $r = \fc{\min(r_i)}{\max(r_i)}$.  

Let $\cal P_i$
be a partition\footnote{We allow overlaps on sets of measure 0.} of the ground set $\Om$, for each $i\in [L]$, with $\cal P_1=\{\Om\}$. 

Define the \vocab{overlap parameter} of $(\cal P_i)_{i=1}^L$ to be 
$$
\de((\cal P_i)_{i=1}^L) = \min_{1< i\le L, A\in \cal P} \ba{
\int_{A} \min\{p_{i-1}(x),p_i(x)\} \dx 
}/p_i(A).
$$
\end{asm}

\begin{thm}
Suppose Assumptions~\ref{asm} hold. Define
$$
p_{\min}=\min_{i,A\in \cal P_i} p_i(A). 
$$

Then the spectral gap of the tempering chain satisfies
\begin{align}
\Gap(M_{\st}) &\ge\fc{r^4\de^2p_{\min}^2}{32L^4}
\min_{1\le i\le L, A\in \cal P_i} (\Gap(M_{i}|_A)).
\end{align} 
\label{t:temperingnochain}
\end{thm}

\begin{proof}
Let $p_{\st}$ be the stationary distribution of $P_{\st}$.
First note that we can easily switch between $p_i$ and $p_{\st}$ using $p_{\st}(A\times \{i\}) = w_i p_i(A)$.  Note $w_i\ge \fc rL$. 

Define the partition $\cal P$ on $\Om\times [L]$ by 
$$
\cal P = \set{A\times \{i\}}{A\in \cal P_i}.
$$

By Theorem~\ref{thm:gap-product},
\begin{align}\label{eq:st-gap-prod0}
\Gap(M_{\st}) & \ge \rc2 \Gap(\ol M_{\st}) 
\min_{B\in \cal P}\Gap(M_{\st}|_B).
\end{align}
The second term $\Gap(M_{\st}|_B)$ is related to $(\Gap(M_{i}|_A))$. We now lower-bound $\Gap(\ol M_{\st}) $. We will abuse notation by considering the sets $B\in \cal P$ as states in $\ol M_{\st}$, and identify a union of sets in $\cal P$ with the corresponding set of states for $\ol{M}_{\st}$. 

We bound $\Gap(\ol M_{\st})$ by bounding the conductance of $\ol M_{\st}$ using Cheeger's inequality (Theorem~\ref{thm:cheeger}). 

Suppose first that $S\subeq \cal P$, $\Om\times \{1\}\nin S$. Intuitively, this means the ``highest temperature'' is not in the set. We will go from top to bottom until we find a partition that is in the set, the interaction of this part and the temperature above it will already provide a large enough cut.

Let $i$ be minimal such that $A\times \{i\}\in S$ for some $A$. There is 
$\rc{2L}$ 
probability of proposing a switch to level $i-1$, so 
\begin{align}
P_{\st}(A\times \{i\}, \Om\times \{i-1\})
&\ge 
\rc{2L} \frac{1}{r_ip_i(A)}
 \int_A r_ip_i(x) \min\bc{
\fc{p_{i-1}(x)}{p_i(x)}\fc{r_{i-1}}{r_i}, 1
}\dx\\
&=
\rc{2L}\int_A  \min\bc{p_{i-1}(x)\fc{r_{i-1}}{r_i}, p_i(x)}\dx/p_i(A)\\
&\ge
\rc{2L} \fc{\min r_j}{\max r_j} \int_A \min\bc{p_{i-1}(x), p_i(x)} \dx /  p_i(A)\\
&\ge 
\rc{2L} r\de.
\label{eq:path-bd2-0}
\end{align}

We have that (defining $Q$ as in Definition~\ref{df:conduct})
\begin{align}
\phi_{\st}(S) & =  \fc{\ol Q_{\st}(S,S^c)}{p(S)}\\
&\ge \fc{p_{\st}(A\times \{i\}) P_{\st}(A\times \{i\}, \Om\times \{i-1\})}{p(S)}\\
&\ge \fc{rp_{\min}}{L}  
\rc{2L}
 r\de
=\fc{r^2\de p_{\min}}{2L^2}
\end{align}

Now consider the case that $\Om\times \{1\}\in S$. This case the highest temperature is in the set. We then try to find the part with highest temperature that is not in the set. 

Note that $p_{\st}(\Om\times \{1\}) \ge \fc{r}L$. Define $A\times \{i\}$ as above for $S^c$, then
\begin{align}
\phi_{\st}(S)&
=\fc{\ol Q_{\st}(S,S^c)}{p(S)}\\
&\ge \fc{p_{\st}(\Om\times \{1\}) P_{\st}(A\times \{i\}, \Om\times \{i-1\})}{p(S)}\\
&\ge \fc{r}{L} 
\rc{2L} r\de.
\end{align}
Thus by Cheeger's inequality~\ref{thm:cheeger},
\begin{align}
\Gap(\ol M_{\st}) &\ge \fc{\Phi(\ol M_{\st})^2}{2}
=\fc{r^4\de^2p_{\min}^2}{8L^4}
\label{eq:gap-proj}
\end{align}

Therefore we have proved the projected Markov chain (between partitions) has good spectral gap. What's left is to prove that inside each partition the Markov chain 
has good spectral gap, note that
\begin{align}
\Gap(M_{\st}|_{B\times \{i\}}) & \ge
\rc 2\Gap(M_i|_A)
\label{eq:gap-rest}
\end{align}
because the chain $M_{\st}$, on a state in $\Om\times \{i\}$, transitions according to $M_i$ with probability $\rc2$. 
Plugging~\eqref{eq:gap-proj} and~\eqref{eq:gap-rest} into~\eqref{eq:st-gap-prod0} gives the bound.
\end{proof}

\begin{rem}
Suppose in the type 2 transition we instead pick $k'$ as follows: With probability $\rc 2$, let $k'=k-1$, and with probability $\rc2$, let $k'=k+1$. If $k'\nin [L]$, let $k'=k$ instead.

Then the $\rc{2L}$ becomes $\rc4$ in the proof above so we get the improved gap
\begin{align}
\Gap(M_{\st}) &\ge\fc{r^4\de^2p_{\min}^2}{128L^2}
\min_{1\le i\le L, A\in \cal P_i} (\Gap(M_{i}|_A)).
\end{align}
\end{rem}

%% file: defining_partitions.tex
\section{Defining the partitions}

In this section, we assemble all the ingredients to show that there exists a partition for the Langevin chain such that $\min_{2\le i\le L, A\in \cal P_i}(\Gap(M_i|_A))$ is large, and each part in the partition also has significant probability. Hence the partition will be sufficient for the partitioning technique discussed in previous section.

The high-level plan is to use Theorem~\ref{thm:gt14} to find the partitions for each temperature. Indeed, if we have a mixture of $n$ gaussians, it is not hard to show that the $(n+1)$-st eigenvalue is large:

\begin{lem}[Eigenvalue gap for mixtures]\label{lem:m+1-eig} 

Let $p_i (x)= e^{-f_i(x)}$ be probability distributions on $\Om$ and let $p(x) = \sumo in \weight_i p_i(x)$,  
where $\weight_1,\ldots, \weight_n>0$ and $\sumo in \weight_i=1$. 
Suppose that for 
each $p_i$, 
a Poincar\'e inequality holds with constant $C
$.

Then the ($n+1$)-th eigenvalue of $\sL$ satisfies
$$
\la_{n+1}(-\sL) \ge \rc C. 
$$
\end{lem}

We defer the proof to Section~\ref{subsec:gap}. 
However, there are still many technical hurdles that we need to deal with before we can apply Theorem~\ref{thm:gt14} on spectral partitioning.
\begin{enumerate}
\item When the temperature is different, the distribution (which is proportional to $e^{-\beta f(x)}$) is no longer a mixture of gaussians. We show that it is still close to a mixture of gaussians in the sense that the density function is point-wise within a fixed multiplicative factor to the density of a mixture of gaussians (Section~\ref{subsec:multiplicative}). This kind of multiplicative guarantee allows us to relate the Poincar\'e constants between the two distributions.
\item We then show (Section~\ref{subsec:smallsetpoincare}) a Poincar\'e inequality for all small sets. This serves two purposes in the proof: (a) it shows that the inner-conductance is large. (b) it shows that if a set has small conductance then it cannot be small. We also deal with the problem of continuous time here by taking a fixed time $t$ and running the Markov chain in multiples of $t$. 
\item Now we can prove Lemma~\ref{lem:any-partition}, which shows that if we discretize the continuous-space Markov chain, then there exists good partitioning in the resulting discrete-space Markov chain (Section~\ref{subsec:main}).
\item We then show that if we restrict the Langevin chain to a large ball, and then discretize the space in the large ball finely enough, then in the limit the spectral gap of the discretized chain is the same as the spectral gap of the continuous-space Markov chain (Section~\ref{subsec:finepartition}).
\item  Finally in Section~\ref{subsec:compactset} we show it is OK to restrict the Langevin chain restricted to a large ball.
\end{enumerate}

\subsection{Proving the eigenvalue gap}
\label{subsec:gap}
Now we prove Lemma~\ref{lem:m+1-eig}. The main idea is to use the variational characterization of eigenvalues, and show that there can be at most $n$ ``bad'' directions. 

\begin{proof}
We use the variational characterization of eigenvalues:
$$
\la_{n+1}(-\sL) = \maxr{\text{subspace }S\subeq L^2(p)}{ \dim S = n}\min_{g\perp_p S}
\fc{-\an{g,\sL g}}{\ve{g}_p^2}.
$$
To lower-bound this, it suffices to produce a $n$-dimensional subspace $S$ and lower-bound $\fc{\int_{\R^d}\ve{\nb g}^2p\dx}{\ve{g}_p^2}$ for $g\perp S$.
We choose
\begin{align}
S &= \spn\set{\fc{p_i}{p}}{1\le i\le n}.
\end{align}
Take $g\perp_p \fc{p_i}{p}$ for each $i$. Then, since a Poincar\'e inequality holds on $p_i$,
\begin{align}
\int_{\R^d} g\fc{p_i}p p\dx&=0\\
\implies
\fc{\int_{\R^d}\ve{\nb g}^2 p_i\dx}{\Var_{p_i}(g)} = 
\fc{\int_{\R^d}\ve{\nb g}^2 p_i\dx}{\ve{g}_{p_i}^2}
&\ge\rc C.
\end{align}
Thus
\begin{align}
\fc{\int_{\R^d}\ve{\nb g}^2p\dx}{\ve{g}_p^2}
& = \fc{\sumo im \int_{\R^d}\ve{\nb g}^2w_ip_i\dx}{\sumo in  w_i \ve{g}_{p_i}^2}\ge \rc C,
\end{align}
as needed.
\end{proof}

\subsection{Scaled temperature approximates mixture of gaussians}
\label{subsec:multiplicative}
The following lemma shows that changing the temperature is approximately the same as changing the variance of the gaussian. We state it more generally, for arbitrary mixtures of distributions in the form $e^{-f_i(x)}$. 

\begin{lem}[Approximately scaling the temperature]\label{lem:close-to-sum}
Let $p_i (x)= e^{-f_i(x)}$ be probability distributions on $\Om$ such that for all $\be>0$, $\int e^{-\be f_i(x)}\dx<\iy$. Let 
\begin{align}
p(x) & = \sumo in \weight_i p_i(x)\\
f(x) &= -\ln p(x)
\end{align}
where $\weight_1,\ldots, \weight_n>0$ and $\sumo in \weight_i=1$. Let $w_{\min}=\min_{1\le i\le n}\weight_i$.

Define the distribution at inverse temperature $\be$ to be $p_\be(x)$, where
\begin{align}
g_\be(x) &= e^{-\be f(x)}\\
Z_\be &= \int_{\Om} e^{-\be f(x)}\dx\\
p_\be(x) &= \fc{g_\be(x)}{Z_\be}.
\end{align}
Define the distribution $\wt p_\be(x)$ by
\begin{align}
\wt g_\be(x) &= \sumo in \weight_i e^{-\be f_i(x)}\\
\wt Z_\be &= \int_{\Om} \sumo in {\weight_i e^{-\be f_i(x)}}\dx\\
\wt p_\be(x) &= \fc{\wt g_\be(x)}{\wt Z_\be}.
\end{align}
Then for $0\le \be\le 1$ and all $x$,
\begin{align}
\label{eq:scale-temp1}
g_\be(x) &\in \ba{1,\rc{w_{\min}}} \wt g_\be\\
\label{eq:scale-temp2}
p_\be(x)& \in \ba{1, \rc{w_{\min}}}\wt p_\be \fc{\wt Z_\be}{Z_\be}\sub \ba{w_{\min}, \rc{w_{\min}}}\wt p_\be.
\end{align}
%
%
\end{lem}
\begin{proof}
By the Power-Mean inequality,
\begin{align}
g_\be(x) &= \pa{\sumo in w_i e^{-f_i(x)}}^\be\\
&\ge \sumo in w_i e^{-\be f_i(x)} = \wt g_\be(x).
\end{align}
On the other hand, given $x$, setting $j=\amin_i f_i(x)$, 
\begin{align}
g_\be(x) & = \pa{\sumo in w_i e^{-f_i(x)}}^\be\\
&\le (e^{-f_j(x)})^{\be}\\
&\le \rc{w_{\min}}\sumo in w_i e^{-\be f_i(x)} = \rc{w_{\min}} \wt g_\be(x).
\end{align}
This gives~\eqref{eq:scale-temp1}. This implies $\fc{\wt Z_\be}{Z_\be} \in [w_{\min},1]$, which gives~\eqref{eq:scale-temp2}.
\end{proof}

\subsection{Poincar\'e inequalities on small subsets}
\label{subsec:smallsetpoincare}
%

In this section we prove Poincar\'e inequalities for small sets. In fact we need to prove that this property is true robustly, in order to transform the continuous time Markov chain to a discrete time Markov chain.

\begin{df}
Given a measure $p$ on $\Om$, say that a Poincar\'e inequality with constant $C$ 
holds on sets of measure $\le D$ if whenever $p(\Supp(g))\le D$, 
$$
\cal E_p(g) = \int_{\Om} \ve{\nb g}^2p\dx\ge \rc C \Var_p(g).
$$
\end{df}

This is robust in the following sense: If the above condition is satisfied, then $g$ still satisfies a Poincar\'e inequality even if it is not completely supported on a small set, but just has a lot of mass on a small set.

The main reason we need the robust version is that when we transform the continuous time Markov chain to a discrete time Markov chain, even if we initialized in a small set after some time the probability mass is going to spill over to a slightly larger set.

\begin{lem}[Robustly Poincar\'e on small sets]\label{lem:poincare-small}
Let $A\subeq \R^d$ be a subset. 
Suppose that for $p$, a Poincar\'e inequality with constant $C$ holds on sets of measure $\le 2p(A)$. Then if 
\begin{align}\label{eq:conc-A}
\int_A g^2 p\dx \ge k\int_{\R^d} g^2p\dx
\end{align}
with $k>2p(A)$, then
$$
\cal E_p(g) \ge \fc{1}{20C}\pa{1-\fc{2p(A)}{k}}k\int_{\R^d} g^2p\dx.
$$
\end{lem}
We lower-bound $\cal E_p(g)$ by showing that \eqref{eq:conc-A} implies that not much of $g$'s mass comes from when $g^2$ is small, so that much of $g$'s mass comes from the intersection of the set $B'$ where $g$ is large and the set $A$. This means we can use the Poincar\'e inequality on a ``sliced'' version of $g$ on $A$.

\begin{proof}
By scaling we may assume $\int g^2p\dx=1$. It suffices to show that $\cal E(g) \ge \fc{1}{20C}\pa{1-\fc{2p(A)}{k}}k$.

Let 
\begin{align}
B &= \set{x\in \Om}{g(x)^2 \ge \fc{k}{2p(A)}}\\
h(x) &=\begin{cases}
0,& g(x)^2 \le \fc{k}{2p(A)}\\
g(x) - \sfc{k}{2p}, &g(x) > \sfc{k}{2p(A)}\\
g(x) + \sfc{k}{2p}, & g(x) < -\sfc{k}{2p(A)}.
\end{cases}
\end{align}
i.e., we ``slice'' out the portion where $g(x)^2$ is large and translate it to 0. (Note we cannot just take $h(x) = g(x) \one_B(x)$ because this is discontinuous. Compare with~\cite[Proposition 3.1.17]{bakry2013analysis}.)
\begin{align}
\cal E_p(g) &= \int\ve{\nb g}^2p\dx \\
&\ge \int \ve{\nb h}^2p\dx = \cal E(h).
\end{align}
By Cauchy-Schwarz, noting that $\Supp(h)\subeq B$ and $\Vol_p(B) \le \fc{2p(A)}{k}$,
\begin{align}
\pa{\int hp\dx}^2 & \le \pa{\int h^2p\dx} \fc{2p(A)}{k}.
\end{align}•
We can lower bound ${\int h^2p\dx}$ as follows. Let $B'=\set{x\in \Om}{g(x)^2 \ge \fc{2k}{3p(A)}}$.
\begin{align}
\int_A g^2p\dx &\ge k\\
\int_{A\cap {B'}^c} g^2 p\dx &\le \fc{2k}{3p(A)}p(A) = \fc{2k}{3}\\
\int_{A\cap B'} g^2 p\dx &\ge k - \fc{2k}3 = \fc k3\\
\int_\Om h^2p\dx &\ge \int_{A\cap B'} \fc{h^2}{g^2} g^2 p\dx \\
&\ge 0.15\int_{A\cap B'} g^2p\dx\\
&\ge \rc{20}k.
\end{align}
where we used the fact that when $y^2 \ge\fc{2k}{2p(A)}$, $\fc{(y-\sfc{k}{2p(A)})^2}{y^2}\ge \fc{\pa{\sfc{2k}{3p} - \sfc{k}{2p}}^2}{\fc{2k}{3p}} >0.15$.
Putting everything together,
\begin{align}
\label{eq:var-lb}
\Var(h) &= \pa{\int_B h^2p\dx} - \pa{\int_B hp\dx}^2 \\
&\ge \pa{1-\fc{2p(A)}{k}}\pa{\int_B h^2p\dx}\\
&\ge \pa{1-\fc{2p(A)}{k}}\fc{k}{20}\\
\cal E_p(g)&\ge \cal E_p(h)\\
&\ge \rc C\Var(h) =\rc C\pa{1-\fc{2p(A)}{k}}\fc{k}{20}.
\end{align}•
\end{proof}


\begin{lem}[Conductance from small-set Poincar\'e]\label{lem:sse}
Let $A\subeq \Om$ be a set.
Suppose that a Poincar\'e inequality with constant $C$ holds on sets of measure $\le 2p(A)$.
Let 
$$\phi_t(A) = \int_A
P_t(x,A^c)
 \fc{p(x)}{p(A)}\dx$$
be the conductance of a set $A$ after running Langevin for time $t$. 
Then
\begin{align}
\phi_t(A) &\ge \min\bc{\rc 2, \fc{1}{80C}\pa{1-4p(A)}t}.
\end{align}•
\end{lem}
A Poincar\'e inequality can be thought of as giving a lower bound on ``instantaneous'' conductance. We show that this implies good conductance for a finite time $T$. What could go wrong is that the rate of mass leaving a set $A$ is large at time 0 but quickly goes to 0. We show using Lemma~\ref{lem:poincare-small} that this does not happen until significant mass has escaped.
\begin{proof}
We want to bound
\begin{align}
\phi_t(A) &= 1-\rc{p(A)}\int 
P_t(x,A)
\one_Ap(x)\dx\\
&=1-\rc{p(A)} \int (P_t\one_A)\one_Ap(x)\dx\\
&=1-\rc{p(A)} \an{P_t\one_A, \one_A}_p\\
&=1-\rc{p(A)} \an{P_{t/2}\one_A, P_{t/2}\one_A}_p
\end{align}
since $(P_t)_{t\ge 0}$ is a one-parameter semigroup and is self-adjoint with respect to $p$.
Now by definition of $\sL$, 
\begin{align}
\ddd t\an{P_t g, P_tg}_p &=
2\an{P_tg,\ddd tP_tg}_p 
=2\an{P_tg, \sL P_tg}_p = 
-2\cal E(P_{t}g)
\end{align}
Let $t_0$ be the minimal $t$ such that 
\begin{align}
\int_A\ve{P_t\one_A}^2 p\dx&\le \rc2 \int_A \ve{\one_A}^2p\dx = \rc2 p(A).
\end{align}
(Let $t=\iy$ if this never happens; however, \eqref{eq:condA} will show that $t<\iy$.)
For $t<2t_0$, by Lemma~\ref{lem:poincare-small},
\begin{align}
\ddd t \int\ve{P_t \one_A}^2p\dx 
&\le -2\cal E(P_t\one_A)\\
&\le -2 \fc{1}{40C}(1-4p(A)) \ddd t \int\ve{P_t \one_A}^2p\dx\\
&\le -\fc{1}{20C}(1-4p(A)) \ddd t \int\ve{P_t \one_A}^2p\dx.
\end{align}
This differential inequality implies exponential decay, a fact known as Gronwall's inequality.
\begin{align}
\int\ve{P_t \one_A}^2p\dx &\le e^{-\fc{C}{20}(1-4p(A))t}\ub{\int_A\ve{\one_A}^2p\dx}{p(A)}\\
\phi_t(A)&\le 1-\rc{P(A)} \ve{P_{t/2}\one_A}_p^2\\
&\le 1-e^{-\fc{1}{40C}(1-4p(A))t}\le \max\pa{\rc2, \fc{1}{80C}(1-4p(A))t}\label{eq:condA}
\end{align}
where the last inequality follows because
because $\ddd xe^{-x}\ge \rc 2$ when $e^{-x}\ge \rc 2$.

For $t\ge 2t_0$, we have, because $\int_A\ve{P_t\one_A}^2p\dx$ is decreasing, $\ve{P_{t/2}\one_A}_p^2\le 
\ve{P_{t_0}\one_A}_p^2=
\rc 2p(A)$ so $\phi_t(A)\ge \rc 2$.
\end{proof}

\begin{lem}[Easy direction of relating Laplacian of projected chain]\label{lem:proj-eig}
Let $(\Om, P)$ be a reversible 
Markov chain and $\ol P$ its projection with respect to some partition $\cal P=\{A_1,\ldots, A_n\}$ of $\Om$. 
Let $\cL = I-P$, $\ol \cL = I-\ol P$. 
$$
\la_n(\cL)\le \la_n(\ol{\cL}).
$$
\end{lem}
\begin{proof}
The action of $\ol{\cal L}$ on functions on $[n]$ is the same as the action of $\cal L$ on the subspace of functions that are constant on each set in the partition, denoted by $L^2(\cal P)$. This means that $L^2([n])$ under the action of $\ol \cL$ embeds into $L^2(p)$ under the action of $\cL$. 
Let $\pi=\E[\cdot |\cal P]:L^2(p)\to L^2(\cal P)$ be the projection; note that for $h\in L^2(\cal P)$, $\an{h,\pi g} = \an{h,g}$. 

By the variational characterization of eigenvalues, we have that for some $S$,
\begin{align}
\la_{n}(\cL) &= \minr{g\perp_p S}{g\in L^2(p)} \fc{\an{g,\cL g}}{\ve{g}_p^2} \le 
\minr{f\perp_p \pi(S)}{g\in L^2(\cal P)}\fc{\an{f,\cL g}}{\ve{g}_p^2}
\le \maxr{S\subeq L^2(\cal P)}{\dim(S)=n-1} \fc{\an{g, \cL g}}{\ve{g}_p^2} = \la_n(\ol\cL).
\end{align}
\end{proof}
\begin{lem}[Small-set Poincar\'e inequality for mixtures]\label{lem:small-poincare}
Keep the setup of Lemma~\ref{lem:close-to-sum}. 
Further suppose that
$$
Z_\be = 
\int_{\R^d} e^{-\be f_1}\dx=\cdots = \int_{\R^d} e^{-\be f_n}\dx.
$$
(E.g., this holds for gaussians of equal variance.)
Suppose that a Poincar\'e inequality holds for each $\wt p_{\be,i} = \fc{e^{-\be f_i}}{Z_\be}$ with constant $C_\be$.

Then on $p_\be$, a Poincar\'e inequality holds with constant $\fc{2C_\be}{w_{\min}}$ on sets of size $\le \fc{w_{\min}^2}2$.
\end{lem}
\begin{proof}
Let $A=\Supp(g)$. Then for all $j$, by Lemma \ref{lem:close-to-sum},
\begin{align}
w_{\min}\wt p_{\be,j}(A) &\le \weight_j \wt p_{\be, j}(A)\le 
\sumo in \wt p_{\be,i}(A) = \wt p_\be(A) 
\le \rc{w_{\min}} p_\be(A)
\le \fc{w_{\min}}2\\
\implies
\wt p_{\be,j}(A) & \le \rc 2.
\end{align}
As in~\eqref{eq:var-lb}, using $\wt p_{\be,j}(A)\le \rc 2$, 
\begin{align}
\Var_{\wt p_{\be, j}}(g) &\ge 
\pa{\int_A g^2\wt p_{\be, j}\dx} - \pa{\int_A g\wt p_{\be, j}\dx}^2\\
&\ge \rc 2 \int_A g^2 \wt p_{\be,j}\dx
\end{align}
Then
\begin{align}
\cE_{\wt p_\be}(f) 
&= \int_{A} \ve{\nb f}^2 \sumo in w_j \wt p_{\be,j}(x)\dx\\
&\ge \rc{C_\be} \sumo in w_i \Var_{\wt p_{\be,j}}(f)\\
&\ge\rc{C_\be}\sumo in w_i \rc 2 \int_A f^2 \wt p_{\be,j}(x) \,dx\\
&\ge \rc{2C_\be} \Var_{\wt p_\be}(f).
\end{align}
Using Lemma~\ref{lem:close-to-sum} and
 Lemma~\ref{lem:poincare-liy}(3), $p_\be$ satisfies a Poincar\'e inequality with constant $\fc{2C_\be}{w_{\min}}$.
\end{proof}


\subsection{Existence of partition}
\label{subsec:main}

Now we are ready to prove the main lemma of this section, which gives the partitioning for any discrete space, discrete time Markov chain. In later subsections we will connect this back to the continuous space Markov chain.

\begin{lem}\label{lem:any-partition}
Let $p(x)\propto e^{-f(x)}$ be a probability density function on $\R^d$. 
Let $C$ and $\mu\le 1$ be such that the Langevin chain on $f(x)$ satisfies $\la_{n+1}(\sL)\ge \rc{C}$, and a Poincar\'e inequality holds with constant $2C$ on sets of size $\le \mu$.

Let $\cal P=\{A_1,\ldots, A_m\}$ be any partition of $B\subeq \R^d$.
Let $( \R^d, P_T)$ be the discrete-time Markov chain where each step is running continuous Langevin for time $T\le \fc C2$, $(B, P_T|_B)$ be the same Markov chain restricted to $B$, and $([m], \ol{P_T|_B})$ is the projected chain with respect to $\cal P$. 

Suppose that $B$ satisfies the following.
\begin{enumerate}
\item
For all $x\in B$, $P_T(x,B^c) \le \fc{T}{1000C}$. 
\item
$\la_{n+1}(I-P_T|_B) \ge \fc 34\pa{\la_{n+1}(I-P_T) - \fc{T}{6C}}$.
\end{enumerate}•

Then there exists $\ell \le n$ and a partition $\cal J$ of $[m]$ into $\ell$ sets  $J_1,\ldots, J_\ell$ such that the following hold.
\begin{enumerate}
\item
$\cal J$ is a $\pa{\Om\pf{T^2}{C^2m^8}, O\pf{T}{C}}$-clustering.
\item
Every set in the partition has measure at least $\fc{\mu}{2}$.
\end{enumerate}
\end{lem}
\begin{proof}
First we show that the $(n+1)$th eigenvalue is large. Note that the eigenvalues of $P_T$ are the exponentials of the eigenvalues of $\sL$.
\begin{align}
\la_{n+1}(I-\ol{P_T|_B}) 
& \ge \la_{n+1}(I - P_T|_B)&\text{Lemma~\ref{lem:proj-eig}}\\
&\ge \fc34\pa{ \la_{n+1}(I-P_T)-\fc{T}{6C}} & \text{assumption}\\
&= \fc34(1-e^{-\la_{n+1}(-\sL)T}-\fc{T}{6C})
\ge \fc34(1-e^{-T/C} - \fc{T}{6C})\\
&\ge \fc34\pa{ \min\bc{\rc 2 , \fc T{2C}} - \fc{T}{6C}} = \fc{T}{4C}
\end{align}
Let $c$ be a small constant to be chosen.
Let $k\le n+1$ be the largest integer so that 
\begin{align}
\la_{k-1}(I-\ol{P_T|_B}) &\le \fc{c^2 T^2}{C^2n^6}
\end{align}
Then
$\la_k(I-\ol{P_T|_B}) >\fc{c^2 T^2}{C^2n^6}$.
By Theorem~\ref{thm:gt14}, for some $1\le \ell \le k-1$, there exists a clustering with parameters 
\begin{align}
\pa{
\Om\pf{\la_k}{k^2}, O(\ell^3\sqrt{\la_{\ell}})
} = 
\pa{\Om\pf{c^2 T^2}{C^2n^8}
,
O\pf{cT}{C}
}
\end{align}

Now consider a set $J$ in the partition. 
Let $A=\bigcup_{j\in J} A_j$.
Suppose by way of contradiction that $p(A)\le \fc{\mu}{2}$.
By Lemma~\ref{lem:sse} and noting $p(A)<\rc 2$, the conductance of $A$ in $\R^d$ is
\begin{align}\label{eq:cond-below}
\phi_T(A)
\ge 
\fc{1}{80(2C)}(1-4p(A))T
\ge 
\fc{T}{320C}.
\end{align}
The conductance of $A$ in $B$ satisfies
\begin{align}
O\pf{cT}{C} \ge \phi_T|_{B}(A).
\end{align}
Now by assumption, $\fc{\int_A P(x,B^c)p(x)\dx}{p(A)}\le \fc{T}{1000C}$, so
\begin{align}\label{eq:cond-above}
O\pf{cT}{C} + \fc{T}{1000C}\ge \phi_T(A)
\end{align}
For $c$ chosen small enough, together \eqref{eq:cond-below} and \eqref{eq:cond-above} give a contradiction.
\end{proof}

\subsection{Making arbitrarily fine partitions}
\label{subsec:finepartition}
%
%

In this section we show when the discretization is fine enough, the spectral gap of the discrete Markov chain approaches the spectral gap of the continuous-space Markov chain.

We will need the following fact from topology.
\begin{lem}[Continuity implies uniform continuity on compact set]\label{lem:unif-cont}
If $(\Om,d)$ is a compact metric space and $g:\Om\to \R$ is continuous, then for every $\ep>0$ there is $\de$ such that for all $x,x'\in \Om$ with $d(x,x')<\de$, 
$$|g(x)-g(x')|<\ep.$$
\end{lem}

We know that the gap of a projected chain is at least the gap of the original chain, $\text{Gap}(\ol M)\ge \text{Gap}(M)$ (Lemma~\ref{lem:proj-eig}). We show that if the size of the cells goes to 0, then the reverse inequality also holds. Moreover, the convergence is uniform in the size of the cells.
\begin{lem}\label{lem:limit-chain}\label{l:uniform_convergence}
Let $M=(\R^d, P')$ be a reversible Markov chain where
the kernel $P':\Om\times \Om\to \R$ is continuous and $>0$ everywhere
and the stationary distribution $p':\Om\to \R$ is a continuous function.

Fix a compact set $\Om\sub \R^n$. 
Then 
$$
\lim_{\de\searrow 0}\inf_{K, \cal P} \fc{\Gap(\ol{M|_K}^{\cal P})}{\Gap(M|_K)} = 1
$$
where the infimum is over all compact sets $K\subeq \Om$ and all partitions $\cal P$ of $K$ composed of sets $A$ with $\diam(A)<\de$. 
\end{lem}
\begin{proof}
By Lemma~\ref{lem:unif-cont} on $\ln P'(x,y) : \Om\times \Om \to \R$ and $p'(x)$,  there exists $\de>0$ such that for all $x,y\in \Om$ such that $d(x,x')<\de$ and $d(y,y')<\de$, 
\begin{align}\label{eq:unif-cont-p}
\fc{P'(x',y')}{P'(x,y)}&\in [e^{-\ep},e^{\ep}]\\
\label{eq:unif-cont-p2}
\fc{p'(x)}{p'(x')} &\in [e^{-\ep},e^{\ep}].
\end{align}
We also choose $\de$ small enough so that for all sets $A$ with diameter $<\de$, $p(A)<\ep$. 

Let $P$ and $p$ denote the kernel and stationary distribution for $M|_{K}$, and let
\begin{align}
P(x,dy) &= \de_x(dy)  P_{\text{rej}}(x) + P_{\text{tr}}(x,y)\,dy,
\end{align}•
where $P_{\text{rej}}(x)=P'(x,K^c)$ and $P_{\text{tr}} = P'$ (the notation is to remind us that this is when a transition succeeds). 
Let 
\begin{align}
P_{\text{acc}}(x) = 1-P_{\text{rej}}(x) = \int_Y P_{\text{tr}}(x,y)\dy.
\end{align}

Consider a partition $\cal P$ all of whose components have diameter $<\de$. Let the projected chain be $\ol{M|_K} = (\ol \Om, \ol P)$ with stationary distribution $\ol p$. We let capital letters $X,Y$ denote the elements of $\ol \Om$, or the corresponding subsets of $\Om$, and $\ol f$ denote a function $\ol\Om\to \R$. 
Let $Q$ be the probability distribution on $\Om\times \Om$ given by $Q(x,dy) = p(x)P(x,dy)$, and similarly define $\ol Q(X,Y) = \ol p(X)\ol P(X,Y)$. Also write 
\begin{align}
\ol P(X,Y) &=\one_{X=Y}\ol P_{\text{rej}}(X) + 
\ol P_{\text{tr}}(X,Y)\\
\ol P_{\text{rej}}(X) &=\int_X p(x)P_{\text{rej}}(x)\dx=1-\ol P_{\text{acc}}(X)\\
\ol P_{\text{tr}}(X,Y) &= \int_X p(x) P_{\text{tr}}(x,y)\dx.
\end{align}

We have
\begin{align}
\label{eq:gap-p}
\text{Gap}(M) &= \inf_{g}\fc{\iint (g(x)-g(y))^2 p(x)P(x,dy)\dx}{2\int (g(x)-\EE_p g(x))^2p(x)\dx}
= \inf_{g}\fc{\EE_{x,y\sim Q} [g(x)-g(y)]^2}{\EE_{x\sim p}[g(x)-\EE_pg]^2}
\\
\text{Gap}(\ol M) &= \inf_{\ol g}\fc{\sum_{X,Y} (\ol g(X)-\ol g(Y))^2 \ol p(X)\ol P(X,Y)\dx\dy}{2\int (\sum_{X}\ol g(X)-\EE_{\ol p} \ol g(X))^2\ol p(\ol X)\dx}
\\
& = \inf_{g}
\fc{\sum_{X,Y} \pa{\EE_p[g|X]- \EE_p[g|Y]}^2 \ol p(X)\ol P(X,Y)}{2\int (\EE_p [g|\cal P](x) - \EE_p[g])^2p(x)\dx}
= \inf_{g}\fc{\EE_{X,Y\sim \ol Q} [\EE_p[g|X] - \EE_p [g|Y]]^2}{\EE_{X\sim  p}[\EE_p [g|X] - \EE_p[g]]^2}
.
\label{eq:gap-ol-p2}
\end{align}
We will relate these two quantities.

Consider the denominator of \eqref{eq:gap-p}. By the Pythagorean theorem, it equals the variation between sets in the partition, given by the denominator of~\eqref{eq:gap-ol-p2}, plus the variation within sets in the partition.
\begin{align}\label{eq:denom-pythag}
\EE_{x\sim p} [g(x)-\EE_p g]^2 &= 
\EE_{x\sim p} [\EE_p[g|\cal P](x) - \EE_p g]^2 
+\EE_{x\sim p} [g(x) - \EE_p[g|\cal P](x)]^2.
\end{align}

We also decompose the numerator of \eqref{eq:gap-p}. First we show that we can approximate $p(x)P_{\text{tr}}(x,y)$ with a distribution where $y$ is independent of $x$ given the set $X$ containing $x$. Using~\eqref{eq:unif-cont-p} and \eqref{eq:unif-cont-p2},
\begin{align}
p(x) \ol P_{\text{tr}}(X,Y) \fc{p(y)}{\ol p(Y)} 
&= \fc{p(x) p(y) \iint_{X\times Y}p(x') P_{\text{tr}}(x',y')\dx'\dy'}{\ol p(X) \ol p(Y)}\\
&\le e^{2\ep}\fc{p(x) P_{\text{tr}}(x,y) \iint_{X\times Y} p(x') p(y')\dx'\dy'}{\ol p(X)\ol p(Y)}\\
& = e^{2\ep}p(x)P_{\text{tr}}(x,y).
\label{eq:decouple}
\end{align}

Let $R$ be the distribution on $\Om\times\Om$ defined as follows:
\begin{align}
X,Y&\sim \ol Q, &
x&\sim p|_X &
y&\sim p|_Y.
\end{align}•
We then have by~\eqref{eq:decouple} that
\begin{align}
\iint_{\Om\times \Om}(g(x)-g(y))^2 p(x)P(x,y)\dx\dy
&= \iint_{\Om\times \Om}(g(x)-g(y))^2 p(x)P_{\text{tr}}(x,y)\dx\dy\\
&\ge e^{-2\ep}\iint_{\Om\times \Om}(g(x)-g(y))^2 p(x) \ol P_{\text{tr}}(X,Y) \fc{p(y)}{\ol p(Y)} \dx\dy
\\
&=e^{-2\ep}\Big[\EE_{(x,y)\sim R} [g(x)-g(y)]^2\\
&\quad 
 - \sum_{X} \iint_{X\times X}[g(x)-g(y)]^2 p(x) \ol P_{\text{rej}}(X)\fc{p(y)}{\ol p(X)}\dx\dy\Big].
\label{eq:pre-pythag}
\end{align}
We use the Pythagorean Theorem: letting $\cal B$ be the $\si$-algebra of $\Om$,
\begin{align}
\EE_R[g(x)-g(y)|\cal P\times \cal B] &= \EE_p[g|X] - g(y),\\
\EE_R[g(x)-g(y)|\cal P\times \cal P] &= \EE_p[g|X] - \EE_p[g|Y].
\end{align}
Then 
\begin{align}
\eqref{eq:pre-pythag}
&= e^{-2\ep} \Big[\EE_{(x,y)\sim R} [(\EE_p[g|X] - \EE_p[g|Y])^2
+ (g(x)-\E[g|\cal P](x))^2 + (g(y) - \E[g|\cal P](y))^2]\\
&\quad  - \sum_X \iint_{X\times X} [(g(x)-\E[g|\cal P](x))^2 + (g(y) - \E[g|\cal P](y))^2] \ol p(X) \ol P_{\text{rej}}(X) \fc{p(y)}{\ol p(Y)}\dx\dy\Big]\\
&= e^{-2\ep} \Big[
\EE_{x,y\sim Q} [g(x)-g(y)]^2\\
&\quad + \iint_{X\times Y} [(g(x)-\E[g|\cal P](x))^2 + (g(y)-\E[g|\cal P](y))^2] \ol p(X) \ol P_{\text{tr}}(X,Y) \fc{p(y)}{\ol p(Y)}\dx\dy
\Big]\label{eq:numer-decomp}
\end{align}

Thus, using $\fc{a'+b'}{a+b}\ge \min\bc{\fc{a'}{a},\fc{b'}{b}}$ for $a',b',a,b>0$, and the decompositions~\eqref{eq:denom-pythag} and~\eqref{eq:numer-decomp},
\begin{align}
\fc{\EE_{x,y\sim Q} [g(x)-g(y)]^2}{2\EE_{x\sim p} [g(x)-\EE_p g]^2}
&\ge e^{-2\ep} \min
\Big\{
\fc{ \EE_{X,Y\sim \ol Q} [(\EE_p [g|X]- \EE_p[g|Y])^2]
}{
2\EE_{x\sim p} [\EE_p [g|\cal P](x) - \EE_pg]^2
},\\
&\quad 
\fc{
\sum_{X,Y} (\iint_{X\times Y} [(g(x)-\E[g|\cal P](x))^2 + (g(y)-\E[g|\cal P](y))^2] \ol p(X) \ol P_{\text{tr}}(X,Y) \fc{p(y)}{\ol p(Y)}\dx\dy)}{
2\EE_{x\sim p} [g(x) - \EE_p[g|\cal P](x)]^2
}\Big\}.
\label{eq:pythag-2}
\end{align}
The first ratio in the minimum is at least $\text{Gap}(\ol P)$ by~\eqref{eq:gap-ol-p2}. 
We now bound the second ratio~\eqref{eq:pythag-2}. 

The numerator of~\eqref{eq:pythag-2} is 
\begin{align}
 &\ge 
\min_X\ol P_{\text{acc}}(X) \EE_{(x,y)\sim R} [(g(x)-\E[g|\cal P](x))^2 + (g(y)-\E[g|\cal P](y))^2]\\
&= 2\min_X \ol P_{\text{acc}}(X) \EE_{x\sim p}[g(x)-\E[g|\cal P](x)]^2.
\end{align}
We claim that $\ol P_{\text{acc}}(X) \ge (1-\ep)\Gap(\ol M)$. Consider $\ol g(Y) = \one_{X=Y}$. 
Then 
\begin{align}
\Gap(\ol P) &\le 
\fc{\EE_{X,Y\sim \ol Q} [\ol g(X)-\ol g(Y)]^2}
{2[\EE_{x\sim \ol p}[\ol g(X)^2] - [\EE_{x\sim \ol p} \ol g(X)]^2]}\\
&\le \fc{2\ol Q(X,X^c)}{2[\ol p(X)-\ol p(X)^2]}\\
&\le \fc{\ol P_{\text{acc}}(X)}{1-\ol p(X)}\\
&\le  \fc{\ol P_{\text{acc}}(X)}{1-\ep}
\end{align}
 
%
Putting everything together, 
\begin{align}
\Gap(M|_K) &\ge e^{-2\ep}\min\{\Gap(\ol{M|_K}), (1-\ep) \Gap(\ol{M|_K})\}.
\end{align}
Combined with Lemma~\ref{lem:proj-eig} ($\Gap(\ol{M|_K})\ge \Gap(M|_K)$) and letting $\ep\searrow0$, this finishes the proof.
\end{proof}


\subsection{Restriction to large compact set}

Finally we show it is OK to restrict to a large compact set. Intuitively this should be clear as the gaussian density functions are highly concentrated around their means.

\label{subsec:compactset}
\begin{lem}\label{lem:rest-large}
Let $p_\be(x)\propto e^{-\be f(x)}$ where $f(x)=-\ln \sumo in w_ip_i(x)$ and $p_i(x)$ is the pdf of a gaussian with mean $\mu_i$. Let $B_R=\set{x\in \R^d}{\ve{x}\le R}$. 

For any $T>0$, for any $\ep_1>0$, there exists 
$R$ such that for any $r\ge R$,
\begin{enumerate}
\item
For any $x\in B_R$, $P_T(x,B^c)\le e^{-\be T/2}$. 
\item For any $f$ with $\Supp(f)\subeq B_R$, 
\begin{align}\fc{\cE_{P_T|_{B_R}}(g)}{\Var_{p|_{B_R}}(g)}
&\ge (1-\ep_1) \pa{\fc{\cE_{P_T}(g)}{\Var_{p}(g)}-e^{-\be T/2}}
\end{align}
\item
For all $m$, $\la_m(P_T|_{B_R})\ge (1-\ep_1)(\la_m(P_T)-e^{-\be T/2})$.
\end{enumerate}
\end{lem}
Note that we can improve the $e^{-\be T/2}$ to arbitrary $\ep_2>0$ with a more careful analysis using martingale concentration bounds, but this weaker version suffices for us.
\begin{proof}
Let $\mu$ be such that $\ve{\mu_i}\le D$ for all $1\le i\le m$. Let $Y_t =\ve{X_t}^2$. By It\^{o}'s Lemma,
\begin{align}
dY_t & = 
2\an{X_t, -\be \nb f(X_t)\,dt + \sqrt 2 \,dB_t} + 2d \cdot dt\\
&=\pa{-2\an{X_t, 
\fc{-\be \sumo in w_i (X_t-\mu_i) e^{-\ve{X_t-\mu_i}^2/2}}{\sumo in w_i e^{-\ve{X_t-\mu_i}^2/2}}} +2 d} dt + \sqrt 8 (X_t-\mu)^*dB_t\\
&= -2\pa{\an{X_t, 
\fc{\be \sumo in w_i [(X_t-\mu)+(\mu-\mu_i)] e^{-\ve{X_t-\mu_i}^2}}{\sumo in w_i e^{-\ve{X_t-\mu_i}^2}}}  +  2d} dt +\sqrt 8 (X_t-\mu)^*dB_t\\
&\le (-2\be Y_t +\be D \ve{X_t}+ 2d)dt + \sqrt 8 (X_t-\mu)^*dB_t\\
&\le \pa{-\be Y_t + \fc{D^2\be}{4}+2d} dt+ \sqrt 8 (X_t-\mu)^*dB_t
\end{align}
Let $C=\fc{D^2\be}4+2d$ and consider the change of variable $Z_t = e^{\be t}(Y_t-\fc{C}{\be})$. Since this is linear in $Y_t$, It\^o's Lemma reduces to the usual change of variables and
\begin{align}
dZ_t &\le \be e^{\be t} \pa{Y_t-\fc{C}{\be}}dt 
+ e^{\be t} ((-\be Y_t + C) dt + \sqrt 8 (X_t-\mu^*)dB_t)\\
&\le \sqrt 8 e^{\be t} (X_t-\mu)^* dB_t.
\end{align}
Suppose $Z_0$ is a point with norm $\le R$. 
By the martingale property of the It\^o integral and Markov's inequality,
\begin{align}
\E\ba{ e^{\be T} \pa{\ve{X_T}^2  -\fc{C}{\be}}}
&= \E Z_T\le Z_0 = \ve{X_0}^2-\fc{C}{\be}\\
\implies
\E \ve{X_T}^2 &\le e^{-\be T} \ve{X_0}^2 + \fc{C}{\be} (1-e^{-\be T})\\
\Pj(\ve{X_T}\ge R) &\le \fc{e^{-\be T}(R^2+\fc{C}{\be}(1-e^{-\be T}))}{R^2}\\
&\le e^{-\be T/2}
\end{align}
for all $R$ large enough.  This shows the first part.

Note that the restricted $P_T|_{B_R}$ operates on functions $g$ with $\Supp(g)\subeq B_R$ as 
\begin{align}
P_T|_{B_R} g(x) &= \int_{B_R} P_T|_{B_R} (x,\dy) g(x)\\
&=\int_{B_R} P_T(x,y) g(x)\dy + P_T(x,B_R^c) g(x)\\
&= \one_{B_R} [P_Tg(x)  + P_T(x,B_R^c) g(x)]
\end{align}
Without loss of generality we may assume that $\E g=0$. (This is unchanged by whether we take $x\sim p$ or $x\sim p|_{B_R}$.)
Then for $R$ large enough, 
\begin{align}
\an{P_T|_{B_R} g, g}_p
& \le \an{P_Tg,g}_p + e^{-\be T/2}\ve{g}_p^2\\
\an{(I-P_T|_{B_R})g,g}_{p|_{B_R}} &\ge
\fc{\an{(I-P_T|_{B_R})g, g}_p}{p(B_R)}\\
&\ge \fc{\an{(I-P_T)g,g} - e^{-\be T/2}\ve{g}^2}{p(B_R)}\\
\fc{\cE_{p_T|_{B_R}}(g)}{\Var_{p|_{B_R}}(g)}
& \ge \rc{p(B_R)}\pa{\fc{\cE_{p_T}(g)}{\Var_{p}(g)} - e^{-\be T/2}}.
\end{align}
Taking $R$ large enough, $\rc{p(B_R)}$ can be made arbitrarily close to $1$. 
The inequality for eigenvalues follows from the variational characterization as in the proof of Lemma~\ref{lem:poincare-liy}.
\end{proof}

%% file: highest_temp.tex
\section{Mixing at the highest temperature}

\begin{df}
For a function $f:\R^d\to \R\cup \{\iy\}$, define the $\al$-\textbf{strongly convex envelope} $\text{SCE}_\al[f]:\R^d\to \R\cup \{\iy\}$ by
\begin{align}
\text{SCE}_\al[f] = \sup_{\trow{g\le f}{g\text{ is $\al$-s.c.}}}g(x)
\end{align}
where ``s.c.'' is an abbreviation for strongly convex. 
\end{df}
\begin{pr}\label{pr:sce}
$\text{SCE}_\al[f]$ is $\al$-strongly convex.
\end{pr}
Note we use the following definition of strongly convex, valid for any (not necessarily differentiable) function $f:\R^d\to \R$,
\begin{align}
\forall t\in [0,1], \quad 
f(tx+(1-t)y) \le tf(x) + (1-t)f(y) - \rc 2 \al t(1-t) \ve{x-y}_2^2
\end{align}
\begin{proof}
Let $t\in [0,1]$. We check
\begin{align}
t\text{SCE}_\al[f](x) + (1-t) \text{SCE}_\al[f](y)
& = t\sup_{\trow{g\le f}{g\text{ is $\al$-s.c.}}} g(x)
+ 
(1-t)\sup_{\trow{g\le f}{g\text{ is $\al$-s.c.}}} g(y)
\\
&\le \sup_{\trow{g\le f}{g\text{ is $\al$-s.c.}}} [tg(x)+(1-t)g(y)]\\
&\le \sup_{\trow{g\le f}{g\text{ is $\al$-s.c.}}}[g(tx+(1-t)y) + \rc 2 \al t(1-t)\ve{x-y}_2^2]\\
&\le \text{SCE}_\al[f](tx+(1-t)y).
\end{align}
\end{proof}
\begin{lem}\label{lem:hitemp}
Let $ f(x) = -\ln\pa{\sumo in \weight_i e^{-\fc{\ve{x-\mu_i}^2}2}}$ where $\weight_1,\ldots, \weight_\al>0$, $\sumo im \weight_i=1$, and $w_{\min} = \min_{1\le i\le m}\weight_i$.
Suppose $\ve{\mu_i}\le D$ for all $1\le i\le m$.

Then there exists a $\rc2$-strongly convex function $g(x)$ such that $\ve{f-g}_{\iy}\le D^2$.
\end{lem}
\begin{proof}
We show that $g = \text{SCE}_{\rc2}[f]$ works. It is $\rc2$-strongly convex by Proposition~\ref{pr:sce}. We show that $g(x) \le f(x) -D^2$. 

Let $x,y\in \R^d$. 
We have
\begin{align}
f(x) = -\ln\pa{ \sumo in \weight_i e^{-\fc{\ve{x-\mu_i}^2}2}}
&= -\ln \sumo in \pa{\weight_i e^{-\fc{\ve{y-\mu_i}^2}2 - \fc{\ve{x-y}^2}{2} + \an{x-y, \mu_i-y}}}\\
&\ge -\ln \sumo in \pa{\weight_i e^{-\fc{\ve{y-\mu_i}^2}2}} - \max_i\ba{-\fc{\ve{x-y}^2}2 + \an{x-y, \mu_i-y}}\\
& = f(y) + \fc{\ve{x-y}^2}4 + \an{x-y,y} + \min_i \ba{\fc{\ve{x-y}^2}4 - \an{x-y},\mu_i}\\
&\ge f(y) + \fc{\ve{x-y}^2}4 + \an{x-y,y}  +\min_i\ba{-\ve{\mu_i}^2}\\
&\ge f(y) + \fc{\ve{x-y}^2}4 + \an{x-y,y}  - D^2
\end{align}
The RHS is a $\rc2$-strongly convex function in $x$ that equals $f(y)-D^2 $ at $x=y$, and is $\le f(x)$ everywhere. Therefore, $\text{SCE}_{\rc 2}[f](y) \ge f(y) -D^2$.
\end{proof}
\begin{lem}\label{lem:hitempmix}
Keep the setup of Lemma~\ref{lem:hitemp}. Then Langevin diffusion on $\be f(x)$ satisfies a Poincar\'e inequality with constant $\fc{ 16e^{2\be D^2}}{\be}$.
\end{lem}
\begin{proof}
Let $g(x)$ be as in Lemma~\ref{lem:hitemp}. Since $\be g(x)$ is $\fc{\be}2$-strongly convex, by Theorem~\ref{thm:bakry-emery} it satisfies a Poincar\'e inequality with constant $\fc{16}{\be}$. Now $\ve{\be f-\be g}_{\iy}\le \be D^2$, so by Lemma~\ref{lem:poincare-liy}, $\be f$ satisfies a Poincar\'e inequality with constant $\fc{16}{\be}e^{2\be D^2}$.
\end{proof}

%% file: discretization.tex
\section{Discretizing the continuous chains} 

As a notational convenience, in the section to follow we will denote $x^* = \mbox{argmin}_{x \in \mathbb{R}^d} \tilde{f}(x)$. 

\begin{lem} Let $p^t, q^t: \mathbb{R}^d \times [L]  \to \mathbb{R}$ be the distributions after running the simulated tempering chain for $t$ steps, where in $p^t$, for any temperature $i \in L$, the Type 1 transitions are taken according to the (discrete time) Markov kernel $P_T$: running Langevin diffusion for time $T$; in $q^t$, the Type 1 transitions are taken according to running $\frac{T}{\eta}$ steps of the discretized Langevin diffusion, using $\eta$ as the discretization granularity, s.t. $\eta \leq \frac{\sigma^2}{2}$.  
Then, 
\begin{align*} \mbox{KL} (p^t || q^t) \lesssim \frac{\eta^2}{\sigma^6} (D^2+d) T t^2 + \frac{\eta^2}{\sigma^6} \max_i \E_{x \sim p^0( \cdot, i)}\|x - x^*\|_2^2 + \frac{\eta}{\sigma^4} d t T  \end{align*}
\label{l:maindiscretize}
\end{lem}

Before proving the above statement, we make a note on the location of $x^*$ to make sense of $\max_i \E_{x \sim p^0( \cdot, i)}\|x - x^*\|_2^2$ Namely, we show:  

\begin{lem}[Location of minimum] Let $x^* = \mbox{argmin}_{x \in \mathbb{R}^d} \tilde{f}(x)$. Then, $\|x^*\| \leq \sqrt{2} D$.  
\label{l:locatemin}
\end{lem} 
\begin{proof} 
Since $\tilde{f}(0) = \frac{D^2}{\sigma^2}$, it follows that $\min_{x \in \mathbb{R}^d} \tilde{f}(x) \leq -D^2/\sigma^2$. However, for any $x$, it holds that 
\begin{align*} \tilde{f}(x) &\geq \min_i \frac{\|\mu_i - x\|^2}{\sigma^2} \\
&\geq \frac{\|x\|^2 - \max_i\|\mu_i\|^2}{\sigma^2} \\
&\geq \frac{\|x\|^2 - D^2}{\sigma^2} \end{align*} 
Hence, if $\|x\| > \sqrt{2} D$, $\tilde{f}(x) >  \min_{x \in \mathbb{R}^d} \tilde{f}(x)$. This implies the statement of the lemma. 
\end{proof} 

We prove a few technical lemmas. First, we prove that the continuous chain is essentially contained in a ball of radius $D$. More precisely, we show:

\begin{lem}[Reach of continuous chain] Let $P^{\beta}_T(X)$ be the Markov kernel corresponding to evolving Langevin diffusion 
\begin{equation*}\frac{dX_t}{\mathop{dt}} = - \beta \nabla \tilde{f}(X_t) + \mathop{d B_t}\end{equation*} 
with $\tilde{f}$ and $D$ are as defined in \ref{eq:A0} for time $T$. Then, 
\begin{equation*}\E[\|X_t - x^*\|^2] \leq \E[\|X_0 - x^*\|^2] + (4 \beta D^2 + 2 d)T \end{equation*} 
\label{l:reachcontinuous}
\end{lem} 
\begin{proof} 
Let $Y_t = \|X_t - x^*\|^2$. By It\^{o}s Lemma, we have 
\begin{equation} d Y_t = -2 \langle X_t - x^*, \beta \sum_{i=1}^n \frac{w_i e^{-\frac{\|X_t - \mu_i\|^2}{\sigma^2}} (X_t - \mu_i)}{\sum_{i=1}^n w_i e^{-\frac{\|X_t - \mu_i\|^2}{\sigma^2}}}  \rangle + 2 d \mathop{dt} + \sqrt{8} \sum_{i=1}^d (X_t)_i \mathop{d(B_i)_t} \label{eq:contdrift1} \end{equation} 
We claim that 
$$- \langle X_t - x^*, X_t - \mu_i \rangle \leq \frac{D}{2}$$
Indeed, 
\begin{align*} 
- \langle X_t - x^*, X_t - \mu_i \rangle &\leq -\|X_t\|^2 + \|X_t\| (\|\mu_i\| + \|x^*\|) + \|x^*\| \|\mu_i\| \\ 
&\leq 4 D^2  
\end{align*} 
where the last inequality follows from $\|\mu_i|\leq D$ and Lemma~\ref{l:locatemin}
Together with \eqref{eq:contdrift1}, we get 
\begin{equation*} d Y_t \leq  \beta 4 D^2 + 2 d \mathop{dt} + \sqrt{8} \sum_{i=1}^d (X_t)_i \mathop{d(B_i)_t}  \end{equation*} 
Integrating, we get 
\begin{equation*} Y_t \leq  Y_0 + \beta 4 D^2  T + 2 d T + \sqrt{8} \int^T_0 \sum_{i=1}^d (X_t)_i \mathop{d(B_i)_t}  \end{equation*} 
Taking expectations and using the martingale property of the It\^{o}  integral, we get the claim of the lemma. 
\end{proof} 

Next, we prove a few technicall bound the drift of the discretized chain after $T/\eta$ discrete steps. The proofs follow similar calculations as those in \cite{dalalyan2016theoretical}.   

We will first need to bound the Hessian of $\tilde{f}$. 
\begin{lem}[Hessian bound] 
$$\nabla^2 \tilde{f}(x) \preceq \frac{2}{\sigma^2} I, \forall x \in \mathbb{R}^d$$ 
\label{l:hessianbound}
\end{lem}
\begin{proof}
For notational convenience, let $p(x) = \sum_{i=1}^n w_i e^{-f_i(x)}$, where $f_i(x) = \frac{(x - \mu_i)^2}{\sigma^2} + \log Z$ and $Z = \int_{\mathbb{R}^d} e^{-f_i(x)} \dx$. Note that $\tilde{f}(x) = - \log p(x)$. The Hessian of $\tilde{f}$ satisfies 
\begin{align*} \nabla^2 {\tilde{f}} 
&= \frac{\sum_i w_i e^{-f_i} \nabla^2 f_i}{p} - \frac{\frac{1}{2} \sum_{i,j} w_i w_j e^{-f_i} e^{-f_j} (\nabla f_i - \nabla f_j)^{\otimes 2}}{p^2} \\
&\preceq \max_i \nabla^2 f_i \preceq \frac{2}{\sigma^2} I \end{align*}
as we need. 



\end{proof}

\begin{lem}[Bounding interval drift] In the setting of this section, let $x \in \mathbb{R}^d, i \in [L]$, and let $\eta \leq \frac{\sigma^2}{2 \alpha}$.
$$\mbox{KL}(P_T(x, i) || \widehat{P_T}(x,i)) \leq \frac{4 \eta^2 }{3 \sigma^6} \left(\|x - x^*\|_2^2) + 2 Td\right) + \frac{d T \eta}{\sigma^4}$$

\Anote{Doesn't seem we need the general $\alpha$ formulation -- setting it to 1 works?}
\label{l:intervaldrift}
\end{lem}    
\begin{proof}
Let $x_j, i \in [0, T/\eta - 1]$ be a random variable distributed as $\widehat{P_{\eta j}}(x,i)$. By Lemma 2 in \cite{dalalyan2016theoretical} and Lemma~\ref{l:hessianbound}
, we have 
\begin{equation*} \mbox{KL}(P_T(x, i) || \widehat{P_T}(x,i)) \leq \frac{\eta^3}{3 \sigma^4} \sum_{k=0}^{T/\eta - 1} \E[\|\nabla f(x^k)\|^2_2] + \frac{d T \eta}{\sigma^4} \end{equation*}
Similarly, the proof of Corollary 4 in \cite{dalalyan2016theoretical} implies that%
\begin{equation*} \eta \sum_{k=0}^{T/\eta -1} \E[\|\nabla f(x^k)\|^2_2] \leq \frac{4}{\sigma^2} \|x - x^*\|^2_2 + \frac{8 T d}{\sigma^2}\end{equation*}

\end{proof} 

Finally, we prove a convenient decomposition theorem for the KL divergence of two mixtures of distributions, in terms of the KL divergence of the weights and the components in the mixture. Concretely:

\begin{lem} Let $w, w': I \to \mathbb{R}$ be distributions over a domain $I$ with full support. Let $p_i, q_i: \forall i \in I$ be distributions over an arbitrary domain. Then: 
$$ \mbox{KL}\left(\int_{i \in I} w_i p_i || \int_{i \in I} w'_i q_i\right) \leq \mbox{KL}(w || w') + \int_{i \in I} w_i \mbox{KL}(p_i || q_i) $$ 
\label{l:decomposingKL}
\end{lem} 

\begin{proof}  
Overloading notation, we will use $KL (a || b) $ for two measures $a,b$ even if they are not necessarily probability distributions, with the obvious definition.  
\begin{align*} 
\mbox{KL}\left(\int_{i \in I} w_i p_i || \int_{i \in I} w'_i q_i\right) &= \mbox{KL}\left(\int_{i \in I} w_i p_i || \int_{i \in I} w_i q_i \frac{w'_i}{w_i}\right) \\ 
&\leq \int_{i \in I} w_i  \mbox{KL}\left( p_i || q_i \frac{w'_i}{w_i}\right) \\ 
&=\int_{i \in I} w_i \log\left(\frac{w_i}{w'_i}\right) + \mbox{KL}(p_i || q_i) \\ 
&=\mbox{KL}(w || w') + \int_{i \in I} w_i \mbox{KL}(p_i || q_i) \end{align*} 

where the first inequality holds due to the convexity of KL divergence. 

\end{proof} 

With this in mind, we can prove the main claim: 

\begin{proof}[Proof of \ref{l:maindiscretize}]
Let's denote by $P_T\left(x, i\right): \mathbb{R}^d \times [L]  \to \mathbb{R}, \forall x \in \mathbb{R}^d, i \in [L]$ the distribution on $\mathbb{R}^d \times [L]$ corresponding to running the Langevin diffusion chain for $T$ time steps on the $i$-th coordinate, starting at $x \times \{i\}$, and keeping the remaining coordinates fixed. Let us define 
 by $\widehat{P_T}\left(x, i\right): \mathbb{R}^d \times [L]  \to \mathbb{R}$ the analogous distribution, except running the discretized Langevin diffusion chain for $\frac{T}{\eta}$ time steps on the $i$-th coordinate.  

Let's denote by $R\left(x, i\right): \mathbb{R}^d \times [L]  \to \mathbb{R}$ the distribution on $\mathbb{R}^d \times [L]$, running the Markov transition matrix corresponding to a Type 2 transition in the simulated tempering chain, starting at $(x,i)$.   

We will proceed by induction. Towards that, we can obviously write   
\begin{align*} 
p^{t+1} &= \frac{1}{2} \left( \int_{x \in \mathbb{R}^d} \sum_{i=0}^{L-1} p^{t}(x,i) P_T(x, i) \right) + \frac{1}{2} \left( \int_{x \in \mathbb{R}^d} \sum_{i=0}^{L-1}  p^{t}(x,i) R(x, i) \right) 
\end{align*} 
and similarly
\begin{align*} 
q^{t+1}(x,i) &= \frac{1}{2} \left( \int_{x \in \mathbb{R}^d} \sum_{i=0}^{L-1} q^{t}(x,i) \widehat{P_T} (x, i) \right) + \frac{1}{2} \left( \int_{x \in \mathbb{R}^d} \sum_{i=0}^{L-1}  q^{t}(x,i) R(x, i) \right) 
\end{align*}
(Note: the $R$ transition matrix doesn't change in the discretized vs continuous version.) 

By convexity of KL divergence, we have
\begin{align*}
\mbox{KL}(p^{t+1} || q^{t+1}) &\leq \frac{1}{2} \mbox{KL}\left( \int_{x \in \mathbb{R}^d} \sum_{i=0}^{L-1} p^{t}(x,i) P_T(x, i) || \int_{x \in \mathbb{R}^d} \sum_{i=0}^{L-1} q^{t}(x,i) \widehat{P_T}(x, i) \right) \\ 
&+ \frac{1}{2} \mbox{KL}\left( \int_{x \in \mathbb{R}^d} \sum_{i=0}^{L-1}  p^{t}(x,i) R(x, i) || \int_{x \in \mathbb{R}^d} \sum_{i=0}^{L-1}  q^{t}(x,i) R(x, i)  \right)  
\end{align*} 

By Lemma~\ref{l:decomposingKL}, we have that 
$$\mbox{KL}\left( \int_{x \in \mathbb{R}^d} \sum_{i=0}^{L-1}  p^{t}(x,i) R(x, i) || \int_{x \in \mathbb{R}^d} \sum_{i=0}^{L-1}  q^{t}(x,i) R(x, i)  \right) \leq \mbox{KL} (p^t || q^t) $$ 
Similarly, by Lemma~\ref{l:intervaldrift} together with Lemma~\ref{l:decomposingKL} we have
\begin{align*} & \mbox{KL}\left( \int_{x \in \mathbb{R}^d} \sum_{i=0}^{L-1} p^{t}(x,i) P_T(x, i) || \int_{x \in \mathbb{R}^d} \sum_{i=0}^{L-1} q^{t}(x,i) \widehat{P_T}(x, i) \right)  \leq \\ & \mbox{KL} (p^t || q^t) + \frac{4 \eta^2}{3 \sigma^6} \left( \max_i \E_{x \sim p^t( \cdot, i)}\|x - x^*\|_2^2 + 2 Td\right) + \frac{d T \eta}{\sigma^4} \end{align*}

By Lemmas~\ref{l:reachcontinuous} and \ref{l:locatemin}, we have that for any $i \in [0, L-1]$, 
\begin{align*} \E_{x \sim p^t( \cdot, i)}\|x - x^*\|_2^2 &\leq  \E_{x \sim p^{t-1}( \cdot, i)}\|x\|_2 + (4 D^2 + 2d) T \end{align*} 
Hence, inductively, we have $\E_{x \sim p^t( \cdot, i)}\|x - x^*\|_2^2 \leq \E_{x \sim p^0( \cdot, i)}\|x - x^*\|_2^2 +  (4 D^2+2d) T t$

Putting together, we have 
\begin{align*} \mbox{KL} (p^{t+1} || q^{t+1}) \leq \mbox{KL}(p^t || q^t) + \frac{4 \eta^2}{3 \sigma^6} \left( \max_i \E_{x \sim p^0( \cdot, i)}\|x - x^*\|_2^2 + (4 D^2 + 2d) T t + 2 Td \right) + \frac{d T \eta}{\sigma^4} \end{align*}

By induction, we hence have 
\begin{align*} \mbox{KL} (p^t || q^t) \lesssim \frac{\eta^2}{\sigma^6} (D^2+d) T t^2 + \frac{\eta^2}{\sigma^6} \max_i \E_{x \sim p^0( \cdot, i)}\|x - x^*\|_2^2 + \frac{\eta}{\sigma^4} d t T \eta  \end{align*}
as we need. 
\end{proof} 

%
%



%% file: estimates_partition.tex
\section{Proof of main theorem}

Before putting everything together, we show how to estimate the partition functions. 
We will apply the following to $g_1(x) = e^{-\be_{\ell}f(x)}$ and $g_2(x) = e^{-\be_{\ell+1} f(x)}$. 
\begin{lem}[Estimating the partition function to within a constant factor]
Suppose that $p_1(x) =\fc{g_1(x)}{Z_1}$ and $p_2(x)=\fc{g_2(x)}{Z_2}$ are probability distributions on $\Om$. 
Suppose $\wt p_1$ is a distribution such that $d_{TV}(\wt p_1, p_1)<\fc{\ep}{2C^2}$, and $\fc{g_2(x)}{g_1(x)}\in [0, C]$ for all $x\in \Om$. Given $n$ samples from $\wt p_1$, define the random variable
\begin{align}
\ol r = \rc{n} \sumo in \fc{g_2(x_i)}{g_1(x_i)}.
\end{align}
Let
\begin{align}
r = \EE_{x\sim p_1}\fc{g_2(x)}{g_1(x)} = \fc{Z_2}{Z_1}
\end{align}
and suppose $r\ge \rc{C}$. 
Then  with probability $\ge 1-e^{-\fc{n\ep^2}{2C^4}}$, 
\begin{align}
\ab{\fc{\ol r}{r}-1}& \le \ep.
\end{align}
\label{l:partitionfunc}
\end{lem}
\begin{proof}
We have that 
\begin{align}
\label{eq:est-Z1}
\ab{\EE_{x\sim \wt p_1} \fc{g_2(x)}{g_1(x)} - \EE_{x\sim p_1}\fc{g_2(x)}{g_1(x)}}
&\le Cd_{TV}(\wt p_1, p_1)\le \fc{\ep}{2C}.
\end{align}
The Chernoff bound gives
\begin{align}
\Pj
\pa{
\ab{r - \EE_{x\sim \wt p_1} \fc{g_2(x)}{g_1(x)}} \ge \fc{\ep}{2C}
}
&\le 
e^{-\fc{n\pf{\ep}{2C}^2}{2\pf{C}2^2}} = 
e^{-\fc{n\ep^2}{2C^4}}.\label{eq:est-Z2}
\end{align}
Combining \eqref{eq:est-Z1} and \eqref{eq:est-Z2} using the triangle inequality,
\begin{align}
\Pj\pa{|\ol r- r|\ge \rc{\ep}C} \le e^{-\fc{n\ep^2}{2C^4}}.
\end{align}
Dividing by $r$ and using $r\ge \rc{C}$ gives the result.
\end{proof}

\begin{lem}\label{lem:delta}
Suppose that $f(x) = -\ln \ba{\sumo in  w_i e^{-\fc{\ve{x-\mu_i}^2}2}}$, $p(x)\propto e^{-f(x)}$, and for $\al\ge 0$ let $p_\al(x) \propto e^{-\al f(x)}$, $Z_\al =\int_{\R^d} e^{-\al f(x)}\dx$. 
Suppose that $\ve{\mu_i}\le D$ for all $i$. 

If $\al<\be$, then
\begin{align}
\ba{
\int_{A} \min\{p_{\al}(x),p_\be(x)\} \dx 
}/p_\be(A)
&\ge
\min_x\fc{p_\al(x)}{p_\be(x)} \ge \fc{Z_\be}{Z_\al}\\
\fc{Z_\be}{Z_\al}&
\in \ba{
\rc2 e^{-2(\be-\al)\pa{D+\rc{\sqrt\al}\pa{\sqrt d + \sqrt{\ln \pf{2}{w_{\min}}}}}^2}, 1}.
\end{align}
Choosing $\be-\al=O\prc{D^2+\fc{d}{\al}+\rc{\al}\ln\prc{w_{\min}}}$, this quantity is $\Om(1)$.
\end{lem}
\begin{proof}
Let $\wt p_\al(x) \propto \sumo in w_i e^{-\al\ve{x-\mu_i}^2/2}$. 

Let 
$C= D +\rc{\sqrt\al} \pa{\sqrt{d}+\sqrt{2\ln \pf{2}{w_{\min}}}}$. Then by Lemma~\ref{lem:close-to-sum}, 
\begin{align}
\Pj_{x\sim p} (\ve{x}\ge C) 
&\le \rc{w_{\min}} \Pj_{x\sim \wt p_{\al}}(\ve{x}\ge C)\\
&\le \rc{w_{\min}}
\sumo in w_i \Pj_{x\sim \wt g_{\al}} (\ve{x}\ge C)\\
&\le  \rc{w_{\min}}\sumo in w_i \Pj_{x\sim N\pa{0,\rc{\sqrt \al} I_d}}(\ve{x}^2 \ge (C-D)^2)\\
&= \rc{w_{\min}} \Pj_{x\sim N(0, I_d)}\ba{\ve{x}^2 \ge \pa{\sqrt{d} + \sqrt{2\ln \pf{2}{\ep w_{\min}}}}^2}\\
&\le \rc{w_{\min}} \Pj_{x\sim N(0, I_d)}\ba{\ve{x}^2 \ge 
d+2\pa{\sqrt{d \ln\pf{2}{\ep w_{\min}}} + \ln \pf{2}{\ep w_{\min}}}}
\\
&\le  \rc{w_{\min}}\fc{w_{\min}}2 =\rc2
\end{align}
using the $\chi_d^2$ tail bound $\Pj_{y\sim \chi_d^2}(y\ge 2(\sqrt{dx}+x))\le e^{-x}$ from \cite{laurent2000adaptive}.

Thus, using $f(x)\ge 0$, 
\begin{align}
\ba{
\int_{A} \min\{p_{\al}(x),p_\be(x)\} \dx 
}/p_\be(A)
&\ge \int_A \min\{\fc{p_{\al}(x)}{p_\be(x)},1\} p_\be(x)\dx\Big/p_\be(A)\\
&\ge \int_A \min\{\fc{Z_\be}{Z_\al} e^{(-\be + \al)f(x)}, 1\} p_\be(x)\dx\Big /p_\be(A)\\
&\ge \fc{Z_\be}{Z_\al}
\\
& = \fc{\int e^{-\be f(x)}\dx}{\int e^{-\al f(x)}\dx}\\
&=\int_{\R^d} e^{(-\be + \al)f(x)}p_\al(x)\dx\\
&\ge \int_{\ve{x}\le D +\rc{\sqrt\al} \pa{\sqrt{d}+\sqrt{2\ln \pf{2}{w_{\min}}}}} e^{(-\be + \al)f(x)}p_\al(x)\dx\\
&\ge \rc2 e^{-(\be-\al)\max_{\ve{x}\le D +\fc2{\sqrt\al} \pa{\sqrt{d}+\sqrt{\ln \pf{2}{w_{\min}}}}}(f(x))}\\
&\ge \rc 2 e^{-2(\be-\al)\pa{D+\rc{\sqrt \al}(\sqrt d + \sqrt{\ln \pf{2}{w_{\min}}})}^2}
\end{align}
\end{proof}

\begin{lem}\label{lem:chi-sq-mixture}
If $p=\sumo in w_i p_i$ where $p_i$ are probability distributions and $w_i>0$ sum to 1, then 
\begin{align}
\chi^2(p||q) &\le \sumo in w_i \chi^2(p_i||q).
\end{align}
\end{lem}
\begin{proof}
We calculate 
\begin{align}
\chi^2(p||q)&=\sumo in \fc{q(x)^2}{\sumo in w_ip_i(x)} \dx -1\\
&\le \int  \pa{\sumo in w_i}\pa{\sum in w_i\fc{q(x)^2}{p_i(x)}}\dx -1\\
& = \sumo in w_i\pa{\int \fc{q(x)^2}{p_i(x)}\dx-1} = \sumo in w_i \chi^2(p_i||q).
\end{align}
\end{proof}

\begin{lem}\label{lem:a1-correct}
Suppose that Algorithm~\ref{a:stlmc} is run on temperatures $0<\be_1<\cdots< \be_\ell\le 1$, $\ell\le L$ with partition function estimates $\wh{Z_1},\ldots, \wh{Z_\ell}$ satisfying
\begin{align}\label{eq:Z-ratio-correct}
\ab{\fc{\wh{Z_i}}{Z_i} - \fc{\wh{Z_1}}{Z_1}}\le \pa{1+\rc L}^{i-1}
\end{align} 
for all $1\le i\le \ell$
and with parameters satisfying
\begin{align}
\label{eq:beta1}
\be_1 &= O\pf{\si^2}{D^2}\\
\label{eq:beta-diff}
\be_i-\be_{i-1} &= 
O\pf{\si^2}{D^2\pa{d+\ln \prc{w_{\min}}}}\\
T&=\Om\pa{D^2\ln \prc{w_{\min}}}\\
t&=\Omega \left(\frac{D^8 \left(d^4 + \ln\left(\frac{1}{w_{\min}}\right)^4 \right)}{\sigma^8 w^4_{\min}} \ln\left(\frac{1}{\epsilon} \frac{D^2 d \ln\left(1/w_{\min}\right)}{\sigma^8 w_{\min}}\right) \max\left(\frac{D^2}{\sigma^2}, \frac{m^{16}}{\ln(1/w_{\min})^4}\right)\right)\\
\eta &= 
O\pf{\ep \si^2}{dtT}.
\end{align}
Let $q^0$ be the distribution $\pa{N\pa{0,\fc{\si^2}{\be_1}}, 1}$ on $\R^d\times [\ell]$. 
The distribution $q^t$ after running for $t$ steps satisfies $
\ve{p-q^t}_1\le \ep
$.

Setting $\ep = O\prc{L}$ above and taking $m=\Om\pa{\ln \prc{\de}}$ samples, with probability $1-\de$ the estimate 
\begin{align}\wh Z_{\ell+1}&=
\wh Z_\ell \pa{\rc{m}\sumo jm e^{(-\be_{\ell+1} + \be_\ell)f_i(x_j)}}
\end{align} also satisfies~\eqref{eq:Z-ratio-correct}.
\end{lem}

\begin{proof}
First consider the case $\si=1$.

Consider simulated tempering $M_{\st}|_{B_R}$: the type 1 transitions are running continuous Langevin for time $T$, and  if a type 1 transition would leave $B_R$, then instead stay at the same location. 
Let $p^t|_{B_R}$ be the distribution of $M_{st}|_{B_R}$ after $t$ steps starting from $p^0$. 
By the triangle inequality,
\begin{align}
\ve{p-q^t}_1 &\le \ve{p-p|_{B_R}}_1
+ \ve{p|_{B_R} - p^t|_{B_R}}_1 
+ \ve{p^t|_{B_R} - p^t}_1
+ \ve{p^t-q^t}_1
\end{align}
Note that $\ve{p-p|_{B_R}}_1$ and $\ve{p^t|_{B_R}-p^t}_1$ approach 0 as $R\to \iy$, so we concentrate on the other two terms.

Let $M_i$ be the chain at inverse temperature $\be_i$.  Consider first $i\ge 2$. 
By Lemma~\ref{lem:rest-large}, for any $\ep>0$ we can choose $R$ such that for all $x\in B_R$, $P_T(x,B^c) \le e^{-\be T/2}$ and
\begin{align}
\la_n(I-P_T|_{B_R}) \ge (1-\ep)(\la_n(I-P_T) - e^{-\be T/2})
\end{align}
For $T=\Om\pa{\ln \prc{w_{\min}}D^2}$, we have $e^{-\be T/2} = o\pf{Tw_{\min}}{D^2}$.
Thus we can ensure
\begin{align}
\la_n(I-P_T|_{B_R}) \ge \fc 34\pa{\la_n(I-P_T) - o\pf{Tw_{\min}}{D^2}}
\end{align}
By Theorem~\ref{thm:bakry-emery}, a Poincar\'e inequality holds for $\be_i f_j$ with constant $\fc{8}{\be_i}\le O(D^2)$. Letting $\sL'$ be the generator for Langevin on $\wt g_{\be_i}(x)=\sumo jn w_i e^{-\be_i f_j(x)}$, 
\begin{align}
\la_{n+1}(\sL)&\ge w_{\min}\la_{n+1}(\sL')&\text{by Lemma~\ref{lem:close-to-sum} and Lemma~\ref{lem:poincare-liy}}\\
& \ge \Om\pf{w_{\min}}{D^2}&\text{by Lemma~\ref{lem:m+1-eig}}
\end{align}
By Lemma~\ref{lem:small-poincare} on $p_\be$, a Poincar\'e inequality holds with constant $O\pf{D^2}{w_{\min}}$. on sets of size $\le \fc{w_{\min}^2}2$.

By Lemma~\ref{lem:limit-chain},
for each $i$, we can choose a partition $\cal Q_i$ of $B_R$ such that for every compact $K$ consisting of a union of sets in $\cal Q_i$, 
\begin{align}
\Gap({M_i|_K}) &\ge (1-\ep)\Gap(\ol{M_i|_K}^{\cal Q_i}).
\end{align}•

All the conditions of Lemma~\ref{lem:any-partition} are satisfied with $C=O\pf{D^2}{w_{\min}}$. We obtain a partition $\cal P_i$ that is a 
$
\pa{\Om\pa{w_{\min}^2(\ln \rc{w_{\min}})^2{m^8}}, O\pf{Tw_{\min}}{D^2}}
$-clustering under the projected chain with each set in the partition having measure at least $\fc{w_{\min}^2}4$ under $p_i$. 
By having chose the partition fine enough, and by Cheeger's inequality~\ref{thm:cheeger}, for each set $A$ in the $\cal P_i$, 
\begin{align}
\Gap(M_{i}|_A)\ge 
(1-\ep)
\Gap(\ol{M_{i}|_{A}}^{\cal P_i})\ge \Om\pf{(\ln \rc{w_{\min}})^4}{m^{16}}.
\end{align}
For the highest temperature, by Lemma~\ref{lem:hitempmix}, we have
\begin{align}
\Gap(M_1) &=\Om\pa{ \be_1 e^{-2\be_1 D^2}} = \Om\prc{D^2}. 
\end{align}

By Lemma~\ref{lem:delta}, since the condition on $\be_i-\be_{i-1}$ is satisfied, $\de((\cal P_i)_{i=1}^\ell)=\Om(1)$. By assumption on $\wh{Z_i}$, $r=\Om(1)$ ($r$ is defined in Assumption~\ref{asm}). 
By Theorem~\ref{t:temperingnochain}, the spectral gap of the simulated tempering chain is
\begin{align}
G:=\Gap(M_{\st}) = \Om\pa{\fc{r^4\de^2p_{\min}^2}{\ell^4}\min\bc{\pf{(\ln \rc{w_{\min}})^4}{m^{16}}, \rc{D^2}}} = \fc{w_{\min}^4}{\ell^4}\min\bc{\pf{{\ln \rc{w_{\min}}}^4}{m^{16}}, \rc{D^2}}.
\end{align}
Choosing $t=\Om\pf{\ln \pf{\ell}{\ep w_{\min}}}G$, we get by Cauchy-Schwarz and~\eqref{eq:gap-mix} that 
\begin{align}\label{eq:chi-sq-final}
\ve{p-q^t}_1&\le \chi^2(p||q^t)\le  (1-G)^t\chi_2(p||q^0) \le e^{-Gt}\chi^2(p||q^0)
=O\pf{\ep w_{\min}}{\ell}\chi^2(p||q^0)
\end{align}
(Note that $G<2-\la_{\max}$ because the chain is somewhat lazy; it stays with probability $\rc{2\ell}$.)
To calculate $\chi^2(q||p_0)$, first note the $\chi^2$ distance between $N(0,\si^2I_d)$ and $N(\mu, \si^2I_d)$ is $\le e^{\ve{\mu}^2/\si^2}$:
\begin{align}
\chi^2(N(0,\si^2I_d), N(\mu,\si^2 I_d))& = \rc{(2\pi\si^2)^{\fc d2}}\int_{\R^d} e^{2(-\fc{\ve{x-\mu}^2}{2\si^2}) + \fc{\ve{x}^2}{2\si^2}}\dx-1\\
&\le \rc{(2\pi\si^2)^{\fc d2}}\int_{\R^d}  e^{(-\fc{\ve{x}^2}2 + 2\an{x,\mu} - 2\ve{\mu}^2)/\si^2} e^{\ve{\mu}^2/\si^2}\dx= e^{\ve{\mu}^2/\si^2}. 
\end{align}•
Then by Lemma~\ref{lem:chi-sq-mixture} and Lemma~\ref{lem:close-to-sum},
\begin{align}
\chi^2(p||q^0) &\le O\pf{\ell}{w_{\min}} \chi^2\pa{\wt p_\be||N\pa{0, \rc{\be_1}I_d}} \\
&= O\pf{\ell}{w_{\min}}\sumo im \chi^2\pa{N\pa{\mu_i,\rc{\be_1}I_d)||N(0, \rc{\be_1}I_d}}\\
&= O\pf{e^{D^2\be_1}\ell}{w_{\min}} =O\pf{\ell}{w_{\min}}. 
\end{align}
Together with~\eqref{eq:chi-sq-final} this gives $\ve{p-q^t}_1\le \fc\ep3$.

For the term $\ve{p^t-q^t}_1$, use Pinsker's inequality and Lemma~\ref{l:maindiscretize} to get
\begin{align}
\ve{p^t-q^t}_1
&\le 
 \sqrt{2\KL(p^t||q^t)}\\
 &=O\pa{ 
\eta^2[(D^2+d) Tt^2 + D^2] +\eta dtT}\le \fc{\ep}3
\end{align}
for 
$\eta = O\pa{\ep \min\bc{\rc{\sqrt T t(D+\sqrt d)}, \rc{dtT}}} = O(\frac{\ep}{dtT})$. 

This gives $\ve{p-q^t}_1 \le \ep$. 

For the second part, setting $\ep=O\prc{\ell L}$ gives that $\ve{p_l - q^t_l}=O\prc{L}$.
By Lemma~\ref{l:partitionfunc}, noting Lemma~\ref{lem:delta} gives $C=O(1)$, 
 after collecting $n=\Om\pa{L^2\ln \prc{\de}}$ samples, with probability $\ge 1-\de$, $\ab{\fc{\wh{Z_{\ell+1}}/\wh{Z_\ell}}{Z_{\ell+1}/Z_\ell}-1}\le \rc L$. 
Set $\wh{Z_{\ell+1}} = \ol r\wh{Z_\ell}$. Then 
$\fc{\wh{Z_{\ell+1}}}{\wh{Z_\ell}} \in [1-\rc L , 1+\rc L] \fc{Z_{\ell+1}}{Z_\ell}$ and 
$\fc{\wh{Z_{\ell+1}}}{\wh{Z_1}} \in 
\ba{\pa{1-\rc L}^\ell, \pa{1+\rc L}^\ell}\fc{Z_{\ell+1}}{Z_1}$.

Now consider general $\si$. We can transform the problem to the problem where $\si=1$ by the change of variables $x\mapsfrom x\si$. This changes $D$ to $\fc{D}{\si}$. Note that running the discretized chain on this transformed problem with step size $\eta$ corresponds to running the discretized chain on the original problem with step size $\eta \si^2$. This is because a step $Y_{t+1} =Y_t - \eta \nb  g(Y_t) \,dt + \sqrt{2 \eta }\xi_k$ with $g(x) = f\pf{x}{\si}$ corresponds to a step 
 $X_{t+1} = X_t - \eta \si\nb g\pf{X_t}{\si} + \sqrt{2 \eta }\si\xi_k
 = X_t - \eta\si^2 \nb f(X_t) +\sqrt{2\eta \si^2}\xi_k$.
\end{proof}
%
%
%

Now we prove the main theorem, Theorem~\ref{thm:main}.
\begin{proof}[Proof of Theorem~\ref{thm:main}]
Choose $\de=\fc{\ep}{2L}$ where $L$ is the number of temperatures. 
Use Lemma~\ref{lem:a1-correct} inductively, with probability $1-\fc{\ep}2$ each estimate satisfies $\fc{\wh{Z_l}}{\wh{Z_1}}\in [\rc e,e]$. Estimating the final distribution within $\fc{\ep}2$ accuracy gives the desired sample.
\end{proof}

\begin{rem}
One reason that the large powers appear in Lemma~\ref{lem:a1-correct} is that we are going between conductance and spectral gap multiple times, and each time we lose a square by Cheeger's inequality. We care about a spectral gap within sets of the partition, but Theorem~\ref{thm:gt14} controls the conductance rather than the spectral gap. It may be possible to tighten the bound by proving a variant of the theorem that controls the spectral gap directly.
\end{rem}

%% file: mc_background.tex
\section{Background on Markov chains}
\label{a:markovchain}
\subsection{Discrete time Markov chains}


\begin{df}
A (discrete time) Markov chain is $M=(\Om,P)$, where $\Om$ is a measure space and $P(x,y)\dy$ is a probability measure for each $x$.
\footnote{For simplicity of notation, in this appendix we consider chains absolutely continuous with respect to $\R^n$, so we use the notation $p(x)\dx$ rather than $d\mu(x)$, and $P(x,y)\dy$ rather than $P(x,dy)$. The same results and definitions apply with the modified notation if this is not the case.}
It defines a random process $(X_t)_{t\in \N_0}$ as follows. If $X_s=x$, then 
\begin{align}
\Pj(X_{s+1}\in A) = P(x,A) :&=\int_A p(x,y)\dy. 
\end{align}

A \vocab{stationary distribution} is $p(x)$ such that if $X_0\sim p$, then $X_t\sim p$ for all $t$; equivalently, $\int_\Om p(x) P(x,y) \dx = p(y)$. 

A chain is \vocab{reversible} if $p(x)P(x,y) = p(y) P(y,x)$. 
\end{df}

\begin{df}
For a discrete-time Markov chain $M=(\Om, P)$, let $P$ operate on functions as
\begin{align}
(Pg)(x) = \E_{y\sim P(x,\cdot)} g(y) = \int_{\Om} g(x)P(x,y)\dy.
\end{align}

Suppose $M=(\Om, P)$ has unique stationary distribution $p$.
Let
$\an{g,h}_p :=\int_{\Om} g(x)h(x)p(x)\dx$ and define the Dirichlet form and variance by
\begin{align}
\cal E_M(g,h) &= \an{g, (I-P)h}_p \\
\Var_p(g) &= \ve{g-\int_{\Om} gp\dx}_p^2
\end{align}
Write $\cal E_M(g)$ for $\cal E_M(g,g)$. 
Define the eigenvalues of $M$, $0=\la_1\le \la_2\le \cdots$ to be the eigenvalues of $I-P$ with respect to the norm $\ved_{p}$. 

Define the spectral gap by
\begin{align}
\Gap(M) &= \inf_{g\in L^2(p)} \fc{\cal E_M(g)}{\Var_p(g)}.
\end{align}
\end{df}
Note that 
\begin{align}
\cal E_M(g) &= \rc 2 \iint_{\Om\times \Om}(g(x)-g(y))^2 p(x)P(x,y)\dx\dy
\end{align}
and that 
\begin{align}
\Gap(M) = \inf_{g\in L_2(p), g\perp_p \one} \fc{\cal E_M(g}{\ve{g}_p^2} = \la_2(I-P).
\end{align}

\begin{rem}
The normalized Laplacian of a graph is defined as $\cL = I-D^{-\rc 2} A D^{-\rc 2}$, where $A$ is the adjacency matrix and $D$ is the diagonal matrix of degrees. 

A change of scale by $\sqrt{p(x)}$ turns $\cL$ into $I-P$, where $P$ has the transition matrix of the random walk of the graph, so the eigenvalues of $\cL$ are equal to the eigenvalues of the Markov chain defined here.
\end{rem}

The spectral gap controls mixing for the Markov chain. Define the $\chi^2$ distance between $p,q$ by
\begin{align}
\chi_2(p||q) &= \int_\Om \pf{q(x)-p(x)}{p(x)}^2p(x)\dx
= \int_\Om \pf{q(x)^2}{p(x)} - 1.
\end{align}
Let $p^0$ be any initial distribution and $p^t$ be the distribution after running the Markov chain for $t$ steps. Then
\begin{align}\label{eq:gap-mix}
\chi_2(p||p^t) \le (1-G')^t \chi(p||p^0)
\end{align}•
where $G'=\min(\la_2, 2-\la_{\max})$.

\subsection{Restricted and projected Markov chains}

Given a Markov chain on $\Om$, we define two Markov chains associated with a partition of $\Om$.
\begin{df}\label{df:assoc-mc}
For a Markov chain $M=(\Om, P)$, and a set $A\subeq \Om$, define the \vocab{restriction of $M$ to $A$} to be the Markov chain $M|_A = (A, P|_{A})$, where
$$
P|_A(x,B) = P(x,B) + \one_B(x) P(x,A^c).
$$
(In words, $P(x,y)$ proposes a transition, and the transition is rejected if it would leave $A$.)

Suppose the unique stationary distribution of $M$ is $p$. 
Given a partition $\cal P = \set{A_j}{j\in J}$, define the \vocab{projected Markov chain with respect to $\cal P$} to be $\ol M^{\cal P} = (J, \ol P^{\cal P})$, where
$$
\ol P^{\cal P} (i,j) = 
\rc{p(A_i)} \int_{A_i}\int_{A_j} P(x,y)p(x)\dx.
$$
(In words, $\ol P(i,j)$ is the ``total probability flow'' from $A_i$ to $A_j$.)

We omit the superscript $\cal P$ when it is clear.
\end{df}

The following theorem lower-bounds the gap of the original chain in terms of the gap of the projected chain and the minimum gap of the restrictioned chains.
\begin{thm}[Gap-Product Theorem\cite{madras2002markov}]\label{thm:gap-product}
Let $M=(\Om, P)$ be a Markov chain with stationary distribution $p$. 

Let $\cal P=\set{A_j}{j\in J}$ be a partition of $\Om$ such that $p(A_j)>0$ for all $j\in J$. 
$$
\rc 2 \Gap(\ol M^{\cal P}) \min_{j\in J}\Gap(M|_{A_j}) \le \Gap(M) \le \Gap(\ol M^{\cal P}).
$$
\end{thm}

\subsection{Conductance and clustering}

\begin{df}\label{df:conduct}
Let $M=(\Om, P)$ be a Markov chain with unique stationary distribution $p$. Let
\begin{align}
Q(x,y) &= p(x) P(x,y)\\
Q(A,B) & = \iint_{A\times B} Q(x,y)\dx\dy.
\end{align}
(I.e., $x$ is drawn from the stationary distribution and $y$ is the next state in the Markov chain.)
Define the \vocab{(external) conductance} of $S$, $\phi_M(S)$, and the \vocab{Cheeger constant} of $M$, $\Phi(M)$, by
\begin{align}
\phi_M(S) & = \fc{Q(S,S^c)}{p(S)}\\
\Phi(M) &= \min_{S\sub \Om, p(S)\le \rc 2}
\phi_M(S).
\end{align}
\end{df}

\begin{df}\label{df:in-out}
Let $M=(\Om,P)$ be a Markov chain on a finite state space $\Om$. 
We say that $k$ disjoint subsets $A_1,\ldots, A_k$ of $\Om$ are a $(\phi_{\text{in}}, \phi_{\text{out}})$-clustering if for all $1\le i\le k$,
\begin{align}
\Phi(M|_{A_i}) &\ge \phi_{\text{in}}\\
\phi_M(A_i)&\le \phi_{\text{out}}.
\end{align}•
\end{df}

\begin{thm}[Spectrally partitioning graphs, \cite{gharan2014partitioning}]\label{thm:gt14}
Let $M=(\Om, P)$ be a reversible Markov chain with $|\Om|=n$ states. Let $0=\la_1\le \la_2\le \cdots\le \la_n$ be the eigenvalues of the Markov chain. 

For any $k\ge 2$, if $\la_k>0$, then there exists $1\le \ell\le k-1$ and a $\ell$-partitioning of $\Om$ into sets $P_1,\ldots, P_\ell$ that is a 
$
(\Om(\la_k/k^2), O(\ell^3\sqrt{\la_\ell}))
$-clustering.
\end{thm}

\begin{proof}
This is \cite[Theorem 1.5]{gharan2014partitioning}, except that they use a different notion of the restriction of a Markov chain $M|_{A_i}$. We reconcile this below.

They consider the Markov chain associated with a graph $G$, and consider the Cheeger constant of the induced graph, in their definition of a $(\phi_{\text{in}},\phi_{\text{out}})$ clustering: $\Phi(G[A_i])\ge \phi_{\text{in}}$.

We can recast our definition in graph-theoretic language as follows: construct a weighted graph $G$ with weight on edge $xy$ given by $p(x)P(x,y)$. Now the restricted chain $M|_{A}$ corresponds to the graph $G|_A$, which is the same as the induced graph $G[A]$ except that we take all edges leaving a vertex $x\in A$ and redraw them as self-loops at $x$.

On the surface, this seems to cause a problem because if we define the volume of a set $S\subeq A$ to be the sum of weights of its vertices, the volume of $S$ can be larger in $G|_A$ than $G[A]$, yet the amount of weight leaving $S$ does not increase.

However, examining the proof in \cite{gharan2014partitioning}, we see that every lower bound of the form $\phi_{G[A]}(S) $ is obtained by first lower-bounding by $\fc{w(S,A\bs S)}{\Vol(S)}$, which is exactly $\phi_{G|_A}(S)$. Thus their theorem works equally well with $G|_A$ instead of $G[A]$.
\end{proof}

Cheeger's inequality relates the conductance with the spectral gap.
\begin{thm}[Cheeger's inequality]\label{thm:cheeger}
Let $M=(\Om, P)$ be a reversible Markov chain on a finite state space and $\Phi=\Phi(M)$ be its conductance. Then
$$
\fc{\Phi^2}2 \le \Gap(P) \le \Phi.
$$
\end{thm}

\subsection{Continuous time Markov processes}
\label{sec:mdo}

A continuous time Markov process is instead defined by $(P_t)_{t\ge 0}$, and a more natural object to consider is the generator. 
\begin{df}
A continuous time Markov process is given by $M=(\Om, (P_t)_{t\ge 0})$ where the $P_t$ define a random proces $(X_t)_{t\ge 0}$ by
$$
\Pj(X_{s+t}\in A) = P_t(x,A) :=\int_A P(x,y)\dy.
$$
Define stationary distributions, reversibility, $P_tf$, and variance as in the discrete case.

Define the \vocab{generator} $\sL$ by
\begin{align}
\sL g &= \lim_{t\searrow 0} \fc{P_t g - g}{t}.
\end{align}
If $p$ is the unique stationary distribution, define
\begin{align}
\cal E_M(g,h) &= -\an{g, \sL h}_p.
\end{align}
The spectral gap is defined as in the discrete case with this definition of $\cal E_M$. The eigenvalues of $M$ are defined as the eigenvalues of $-\sL$.\footnote{Note that $\cL=I-P$ in the discrete case corresponds to $-\sL$ in the continuous case.} 
\end{df}
Note that in order for $(P_t)_{t\ge 0}$ to be a valid Markov process, it must be the case that $P_tP_u g = P_{t+u}g$, i.e., the $(P_t)_{t\ge 0}$ forms a \vocab{Markov semigroup}.

\begin{df}
A continuous Markov process satisfies a Poincar\'e inequality with constant $C$ if
\begin{align}
\cal E_M(g) \ge \rc C\Var_p(g).
\end{align}
\end{df}
This is another way of saying that $\Gap(M)\ge \rc{C}$. 

For Langevin diffusion with stationary distribution $p$, 
\begin{align}
\cal E_M(g) &= \ve{\nb g}_p^2. 
\end{align}
Since this depends in a natural way on $p$, we will also write this as $\cal E_p(g)$. A Poincar\'e inequality for Langevin diffusion thus takes the form
\begin{align}
\cal E_p(g) = \int_{\Om} \ve{\nb g}^2 p\dx &\ge \rc C \Var_p(g).
\end{align}

We have the following classical result. 

\begin{thm}[\cite{bakry2013analysis}]\label{thm:bakry-emery}
Let $g$ be $\rh$-strongly convex and differentiable. 
Then $g$ satisfies the Poincar\'e inequality
$$
\cal E_p(g) \ge \rh \Var_p(g).
$$

\end{thm}
In particular, this holds for $g(x)={\fc{\ve{x-\mu}^2}{2}}$ with $\rh = 1$, giving a Poincar\'e inequality for the gaussian distribution.

A spectral gap, or equivalently a Poincar\'e inequality, implies rapid mixing (cf. \eqref{eq:gap-mix}):
\begin{align}
\ve{g - P_t g}_2\le e^{-t\Gap(M)} = e^{-\fc tC}.
\end{align}

%% file: example.tex
\section{Examples}
\label{sec:examples}
It might be surprising that sampling a mixture of gaussians require a complicated Markov Chain such as simulated tempering. However, many simple strategies seem to fail. 

\paragraph{Langevin with few restarts} One natural strategy to try is simply to run Langevin a polynomial number of times from randomly chosen locations. While the time to ``escape'' a mode and enter a different one could be exponential, we may hope that each of the different runs ``explores'' the individual modes, and we somehow stitch the runs together. The difficulty with this is that when the means of the gaussians are not well-separated, it's difficult to quantify how far each of the individual runs will reach and thus how to combine the various runs.  

\paragraph{Recovering the means of the gaussians} Another natural strategy would be to try to recover the means of the gaussians in the mixture by performing gradient descent on the log-pdf with a polynomial number of random restarts. The hope would be that maybe the local minima of the log-pdf correspond to the means of the gaussians, and with enough restarts, we should be able to find them. 

Unfortunately, this strategy without substantial modifications also seems to not work: for instance, in dimension $d$, consider a mixture of $d+1$ gaussians, $d$ of them with means on the corners of a $d$-dimensional simplex with a side-length substantially smaller than the diameter $D$ we are considering, and one in the center of the simplex. In order to discover the mean of the gaussian in the center, we would have to have a starting point extremely close to the center of the simplex, which in high dimensions seems difficult.

Additionally, this doesn't address at all the issue of robustness to perturbations. Though there are algorithms to optimize ``approximately'' convex functions, they can typically handle only very small perturbations. \cite{belloni2015escaping, risteski2016algorithms}
   

\paragraph{Gaussians with different covariance} Our result requires all the gaussians to have the same variance. This is necessary, as even if the variance of the gaussians only differ by a factor of 2, there are examples where a simulated tempering chain takes exponential time to converge \cite{woodard2009sufficient}. Intuitively, this is illustrated in Figure~\ref{figure:variance}. The figure on the left shows the distribution in low temperature \--- in this case the two modes are separate, and both have a significant mass. The figure on the right shows the distribution in high temperature. Note that although in this case the two modes are connected, the volume of the mode with smaller variance is much smaller (exponentially small in $d$). Therefore in high dimensions, even though the modes can be connected at high temperature, the probability mass associated with a small variance mode is too small to allow fast mixing.

\begin{figure}[h!]
\centering
\includegraphics[height=2in]{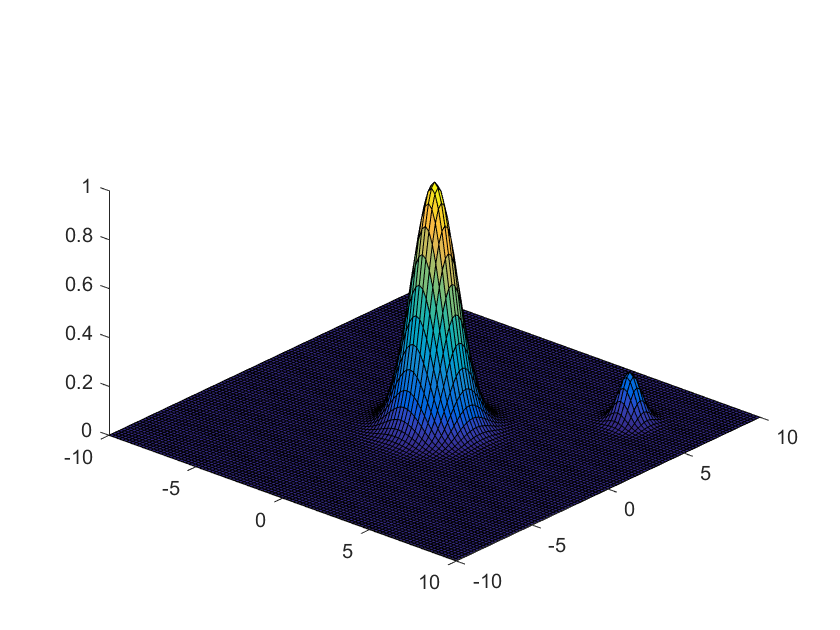}
\includegraphics[height=2in]{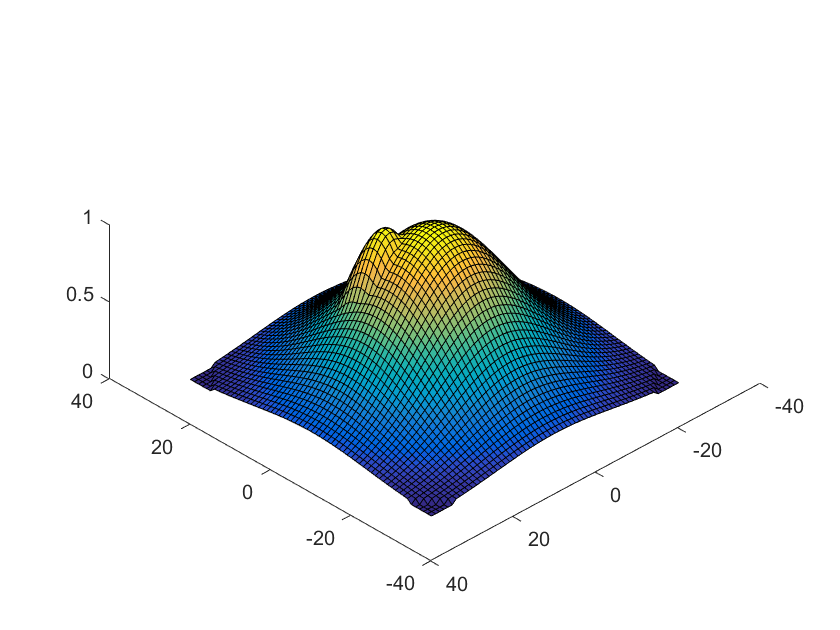}
\caption{Mixture of two gaussians with different covariance at different temperature}
\label{figure:variance}
\end{figure}

%% file: tolerance_perturbation.tex
\section{Pertubation tolerance} \label{sec:perturb}

In the previous sections, we argued that we can sample from distributions of the form $\tilde{p}(x) \propto \exp(\tilde{f(x)})$, where $\tilde{f}$ is as \eqref{eq:tildef}. In this section, the goal we be to argue that we can sample from distributions of the form $p(x) \propto \exp(f(x))$, where $f$ is as \eqref{eq:A0}. 

Our main theorem is the following. 
\begin{thm}[Main theorem with perturbations]\label{thm:main-perturb}
Suppose $f(x)$ and $\tilde{f}(x)$ satisfy \eqref{eq:tildef} and \eqref{eq:A0}. Then, algorithm~\ref{a:mainalgo} with parameters given by Lemma~\ref{lem:a1-correct-perturb} produces a sample from a distribution $p'$ with $\ve{p-p'}_1\le \ep$ in time $\poly\pa{w_{\min}, D, d, \frac{1}{\ep}, e^{\dellarge}, \delsmall}$.
\end{thm}

The theorem will follow immediately from Lemma \ref{lem:a1-correct-perturb}, which is a straightforward analogue of \ref{lem:a1-correct}. More precisely: 

\begin{lem}\label{lem:a1-correct-perturb}
Suppose that Algorithm~\ref{a:stlmc} is run on temperatures $0<\be_1<\cdots< \be_\ell\le 1$, $\ell\le L$ with partition function estimates $\wh{Z_1},\ldots, \wh{Z_\ell}$ satisfying
\begin{align}\label{eq:Z-ratio-correct}
\ab{\fc{\wh{Z_i}}{Z_i} - \fc{\wh{Z_1}}{Z_1}}\le \pa{1+\rc L}^{i-1}
\end{align} 
for all $1\le i\le \ell$
and with parameters satisfying
\begin{align}
\label{eq:beta1-perturb}
\be_1 &= O\left(\min\left(\frac{\si^2}{D^2}, \dellarge \right)\right)\\
\label{eq:beta-diff-perturb}
\be_i-\be_{i-1} &= 
O \left(\min\left( \frac{\si^2}{D^2\pa{d+\ln \prc{w_{\min}}}} , \dellarge \right)\right)\\
T&=\Om\pa{D^2\ln \prc{w_{\min}}}\\
t&=\Omega \left(\frac{D^8 \left(d^4 + \ln\left(\frac{1}{w_{\min}}\right)^4 \right)}{\sigma^8 w^4_{\min}} \ln\left(\frac{1}{\epsilon} \max\left(\frac{D^2 d \ln\left(1/w_{\min}\right)}{\sigma^8 w_{\min}}, e^{\Delta}\right)\right) \max\left(\frac{D^2}{\sigma^2}, \frac{m^{16}}{\ln(1/w_{\min})^4}\right)\right)\\
\eta &= 
O\pf{\ep \si^2}{dtT \delsmall^2}.
\end{align}
Let $q^0$ be the distribution $\pa{N\pa{0,\fc{\si^2}{\be_1}}, 1}$ on $\R^d\times [\ell]$. 
The distribution $q^t$ after running for $t$ steps satisfies $
\ve{p-q^t}_1\le \ep
$.
Setting $\ep = O\prc{L}$ above and taking $m=\Om\pa{\ln \prc{\de}}$ samples, with probability $1-\de$ the estimate 
\begin{align}\wh Z_{\ell+1}&=
\wh Z_\ell \pa{\rc{m}\sumo jm e^{(-\be_{\ell+1} + \be_\ell)f_i(x_j)}}
\end{align} also satisfies~\eqref{eq:Z-ratio-correct}.
\end{lem}

The way we prove this theorem is to prove the tolerance of each of the proof ingredients to perturbations to $f$.

\subsection{Mixing time of the tempering chain} 

We first show that the mixing time of the tempering chain that uses the continous Langevin transition $P_T$ for $p(x) \propto \exp(f(x))$ is comparable to that of $\tilde{p}(x) \propto \exp(\tilde{f}(x))$. Keeping in mind the statement of Lemma~\ref{lem:any-partition}, the following lemma suffices:  

\begin{lem}\label{lem:poincare-liy}
Suppose $\ve{f_1-f_2}_\iy\le \frac{\dellarge}{2}$ and $p_1\propto e^{-f_1}$, $p_2\propto e^{-f_2}$ are probability distributions on $\R^d$. Then the following hold.
\begin{enumerate}
\item
\begin{align}
\fc{\cE_{p_1}(g)}{\ve{g}_{p_1}^2} \ge e^{-\dellarge} \fc{\cE_{p_2}(g)}{\ve{g}_{p_2}^2}.
\end{align}
\item
Letting $\sL_1,\sL_2$ be the generators for Langevin diffusion on $p_1,p_2$,
\begin{align}
\la_n(-\sL_1) &\ge e^{-\dellarge} \la_n(-\sL_2).
\end{align}
\item
If a Poincar\'e inequality holds for $p_1$ with constant $C$, then a Poincar\'e inequality holds for $p_2$ with constant $Ce^{\dellarge}$.
\end{enumerate}•
\end{lem}
Note that if we are given probability distributions $p_1,p_2$ such that $p_1\in [1,e^{\dellarge}]p_2R$ for some $R$, then the conditions of the lemma are satisfied.
\begin{proof}
\begin{enumerate}
\item
The ratio between $p_1$ and $p_2$ is at most $e^{\Delta}$, so 
\begin{align}
\fc{\int_{\R^d} \ve{\nb g}^2p_1\dx}{\int_{\R^d} \ve{g}^2 p_1\dx}
&\ge 
\fc{ e^{-\dellarge}\int_{\R^d} \ve{\nb g}^2p_2\dx}{e^{\dellarge}\int_{\R^d} \ve{g}^2 p_2\dx}
\end{align}
\item
Use the first part along with the variational characterization \begin{align}\la_{m}(-\sL) = \maxr{\text{closed subspace }S\subeq L^2(p)}{\dim(S^{\perp})=m-1} \min_{g\in S} \fc{-\an{g,\sL g}}{\ve{g}_p^2}.\end{align}
\item
Use the second part for $m=2$; a Poincar\'e inequality is the same as a lower bound on $\la_2$.
\end{enumerate}•

\end{proof}

\subsection{Mixing time at highest temperature} 

We show that we can use the same highest temperature corresponding to $f(x)$ for $\tilde{f}(x)$ as well, at the cost of $e^{\dellarge}$ in the mixing time. Namely, since $\|\tilde{f} - f\|_{\infty} \leq \dellarge$, from Lemma \ref{lem:hitemp}, we immediately have: 
\begin{lem} 
If $f$ and $\tilde{f}$ satisfy \eqref{eq:A0} and \eqref{eq:tildef}, there exists a 1/2 strongly-convex function $g$, s.t. $\|f-g\|_{\infty} \leq D^2 + \dellarge$. 
\end{lem} 

As a consequence, the proof of Lemma~\ref{lem:hitempmix} implies 
\begin{lem}  
If $f$ and $\tilde{f}$ satisfy \eqref{eq:A0} and \eqref{eq:tildef}, Langevin diffusion on $\be f(x)$ satisfies a Poincar\'e inequality with constant $\fc{ 16e^{2\be (D^2 + \dellarge)}}{\be}$.
\label{l:hitempmix-perturb}
\end{lem}

\subsection{Discretization} 

The proof of Lemma~\ref{l:reachcontinuous}, combined with the fact that $\ve{\nabla \tilde{f} - \nabla f}_{\infty} \leq \Delta$ gives

\begin{lem}[Perturbed reach of continuous chain] Let $P^{\beta}_T(X)$ be the Markov kernel corresponding to evolving Langevin diffusion 
\begin{equation*}\frac{dX_t}{\mathop{dt}} = - \beta \nabla f(X_t) + \mathop{d B_t}\end{equation*} 
with $f$ and $D$ are as defined in \ref{eq:A0} for time $T$. Then, 
\begin{equation*}\E[\|X_t - x^*\|^2] \lesssim \E[\|X_0 - x^*\|^2] + (\beta (D+ \delsmall)^2  + d)T \end{equation*} 
\end{lem} 
\begin{proof} 
The proof proceeds exactly the same as Lemma~\ref{l:reachcontinuous}, noting that $\ve{\nabla \tilde{f} - \nabla f}_{\infty} \leq \delsmall$  implies
$$ - \langle X_t - x^*, X_t - \mu_i \rangle \leq -\|X_t\|^2 + \|X_t\| (\|\mu_i\| + \|x^*\| + \delsmall) + \|x^*\| (\|\mu_i\| + \delsmall) $$  
\end{proof} 

Furthermore, since $\nabla^2 \tilde{f}(x) \preceq \nabla^2 f(x) + \delsmall I, \forall x \in \mathbb{R}^d$, from Lemma~\ref{l:hessianbound}, we get

\begin{lem}[Perturbed Hessian bound] 
$$\nabla^2 f(x) \preceq \left(\frac{2}{\sigma^2} + \delsmall\right)I , \forall x \in \mathbb{R}^d$$ 
\end{lem}

As a consequence, the analogue of Lemma~\ref{l:intervaldrift} gives: 
\begin{lem}[Bounding interval drift] In the setting of Lemma~\ref{l:intervaldrift}, let $x \in \mathbb{R}^d, i \in [L]$, and let $\eta \leq \frac{(\frac{1}{\sigma} + \delsmall)^2}{\alpha}$. Then,
$$\mbox{KL}(P_T(x, i) || \widehat{P_T}(x,i)) \lesssim \frac{\eta^2 (\frac{1}{\sigma^2} + \delsmall)^3 \alpha}{ 2\alpha - 1} \left(\|x - x^*\|_2^2) + Td\right) + d T \eta\left(\frac{1}{\sigma^2} + \delsmall\right)^2$$
\end{lem}    

Putting these together, we get the analogue of Lemma~\ref{l:maindiscretize}: 
\begin{lem}  Let $p^t, q^t: \mathbb{R}^d \times [L]  \to \mathbb{R}$ be the distributions after running the simulated tempering chain for $t$ steps, where in $p^t$, for any temperature $i \in L$, the Type 1 transitions are taken according to the (discrete time) Markov kernel $P_T$: running Langevin diffusion for time $T$; in $q^t$, the Type 1 transitions are taken according to running $\frac{T}{\eta}$ steps of the discretized Langevin diffusion, using $\eta$ as the discretization granularity, s.t. $\eta \leq \frac{1}{2\left(\frac{1}{\sigma^2} + \delsmall\right)}$.  
Then, 
\begin{align*} \mbox{KL} (p^t || q^t) \lesssim \eta^2 \left(\frac{1}{\sigma^2} + \delsmall\right)^3 \left((D + \delsmall)^2+d\right) T t^2 + \eta^2 \left(\frac{1}{\sigma^2} + \delsmall\right)^3  \max_i \E_{x \sim p^0( \cdot, i)}\|x - x^*\|_2^2 + \eta\left(\frac{1}{\sigma^2} + \delsmall\right)^2  d t T  \end{align*} 
\label{l:maindiscretize-perturb}
\end{lem} 

\subsection{Putting things together} 

Finally, we prove Theorem \ref{lem:a1-correct-perturb}
\begin{proof}[Proof of \ref{lem:a1-correct-perturb}] 
 
The proof is analogous to the one of Lemma~\ref{lem:a1-correct} in combination with the Lemmas from the previous subsections. 

For the analysis of the simulated tempering chain, consider the same partition $\cal P_i$ we used in Lemma~\ref{lem:a1-correct}. 
Then, by Lemma~\ref{lem:poincare-liy}, 
\begin{align}
\Gap(M_{i}|_A) \ge \Omega \left( e^{-\dellarge}\pf{(\ln \rc{w_{\min}})^4}{m^{16}} \right).
\end{align} 

For the highest temperature, by Lemma~\ref{l:hitempmix-perturb}, we have
\begin{align}
\Gap(M_1) &=\Om\left(\be_1 e^{-2\be_1 \left(D^2 + \dellarge\right)}\right) = \Omega(\min(\frac{1}{\dellarge},\frac{1}{D^2})). 
\end{align}

Furthermore, by Lemma~\ref{lem:delta}, since the condition on $\be_i-\be_{i-1}$ is satisfied, $\de((\cal P_i)_{i=1}^\ell)=\Om(1)$. Then, same as in Lemma~\ref{lem:a1-correct}, the spectral gap of the simulated tempering chain 
\begin{align}
G:=\Gap(M_{\st}) =  e^{-\dellarge} \fc{w_{\min}^4}{\ell^4}\pf{{\ln \left( \rc{w_{\min}}\right)}^4}{m^{16}}.
\end{align}
As in Lemma~\ref{lem:a1-correct}, since $t=\Om\pf{\ln (\frac{1}{\epsilon} \max(\frac{l}{w_{\min}}, e^{\dellarge}))}G$, 
\begin{align}\label{eq:chi-sq-final}
\ve{\tilde{p}-q^t}_1&= O\pf{\ep w_{\min}}{\ell}\chi^2(\tilde{p}||q^0)
\end{align}

By triangle inequality,  
\begin{align*} 
\chi^2(\tilde{p}||q^0) &\leq \chi^2(\tilde{p}||p) + \chi^2(p||q^0)
\end{align*}  

The proof of Lemma~\ref{lem:a1-correct} bounds $\chi^2(p||q^0) = O\pf{\ell}{w_{\min}} $, and 
\begin{align*} \chi^2(\tilde{p}||p) &= \int_{x \in \mathbb{R}^d} \left( \frac{\tilde{p}(x) - p(x)}{p(x)} \right)^2 p(x) dx \\ 
&\leq \left( \frac{e^{\dellarge} p(x) - p(x)}{p(x)} \right)^2 p(x)  \\ 
&\leq e^{\dellarge} \end{align*}

From this, we get $\ve{\tilde{p}-q^t}_1 \leq \frac{\epsilon}{3}$. 

For the term $\ve{p^t-q^t}_1$, use Pinsker's inequality and Lemma~\ref{l:maindiscretize-perturb} to get
\begin{align}
\ve{\tilde{p}^t-q^t}_1 \le \sqrt{2\KL(\tilde{p}^t||q^t)} \le \fc{\ep}3
\end{align}
for 
$\eta = O\pa{\ep \min\bc{\rc{\sqrt T t \delsmall^{3/2} (D+\sqrt d + \delsmall)}, \rc{\delsmall^2 dtT}}} = O(\frac{\ep}{\delsmall^2 dtT})$. 

This gives $\ve{\tilde{p}-q^t}_1 \le \ep$. 

The proof of the second part of the Lemma proceeds exactly as \ref{l:maindiscretize-perturb}.

\end{proof}

%% file: other_simulated_tempering.tex
\section{Another lower bound for simulated tempering}
\label{app:other}

\begin{thm}[Comparison theorem using canonical paths, \cite{diaconis1993comparison}]\label{thm:can-path}
Let $(\Om, P)$ be a finite Markov chain with stationary distribution $p$.

Suppose each pair $x,y\in \Om$, $x\ne y$ is associated with a path $\ga_{x,y}$. Define the congestion to be
$$
\rh(\ga) = \max_{z,w\in \Om, z\ne w} \ba{
\fc{\sum_{\ga_{x,y}\ni (z,w) }|\ga_{x,y}|p(x)p(y)}{p(z)P(z,w)}
}.
$$
Then
$$
\Gap(P) \ge \rc{\rh(\ga)}.
$$
\end{thm}

\begin{df}
Say that partition $\cal P$ refines $\cal Q$, written $P\sqsubseteq \cal Q$, if for every $A\in \cal P$ there exists $B\in \cal Q$ such that $A\subeq B$. 

Define a chain of partitions as $\{\cal P_i = \{A_{i,j}\}\}_{i=1}^{L}$, where each $\cal P_i$ is a refinement of $\cal P_{i-1}$:
$$
\cal P_{L} \sqsubseteq \cdots \sqsubseteq \cal P_1.
$$
\end{df}


\begin{thm}\label{thm:sim-temp}
Suppose Assumptions~\ref{asm} hold.

Furthermore, suppose that $(\cal P_i)_{i=1}^L$ is a chain of partitions. 
Define $\ga$ for the chain of partitions as 
$$
\ga((\cal P_i)_{i=1}^{L}) =\min_{1\le i_1\le i_2\le L}\min_{A\in \cal P_{i_1}}
\fc{p_{{i_1}}(A)}{p_{{i_2}}(A)}.
$$

%
%
%
Then 
\begin{align}
\Gap(M_{\st}) &\ge\fc{r^2\ga \de}{32L^3} \min_{1\le i\le L, A\in \cal P_i} (\Gap(M|_A)).
\end{align}
\end{thm}
\begin{proof}
Let $p_{\st}$ be the stationary distribution of $P_{\st}$.
First note that we can easily switch between $p_i$ and $p_{\st}$ using $p_{\st}(A\times \{i\}) = r_i p_i(A)$. 

Define the partition $\cal P$ on $\Om\times \{0,\ldots, l-1\}$ by 
$$
\cal P = \set{A\times \{i\}}{A\in \cal P_i}.
$$

By Theorem~\ref{thm:gap-product},
\begin{align}\label{eq:st-gap-prod}
\Gap(M_{\st}) & \ge \rc2 \Gap(\ol M_{\st}) 
\min_{B\in \cal P}\Gap(M_{\st}|_B).
\end{align}
We now lower-bound $\Gap(\ol M_{\st}) $. We will abuse notation by considering the sets $B\in \cal P$ as states in $\ol M_{\st}$, and identify a union of sets in $\cal P$ with the corresponding set of states for $\ol{M}_{\st}$. 

Consider a tree with nodes $B\in \cal P$, and edges connecting $A\times\{i\}$, $A'\times\{i-1\}$ if $A\in A'$. Designate $\Om\times \{1\}$ as the root. For $X,Y\in \cal P$, define the canonical path $\ga_{X,Y}$ to be the unique path in this tree.

Note that $|\ga_{X,Y}|\le 2(L-1)$. Given an edge $(A\times\{i\},A'\times\{i-1\})$, consider
\begin{align}\label{eq:st-path}
\fc{\sum_{\ga_{X,Y}\ni(A\times\{i\},A'\times\{i-1\})} |\ga_{X,Y}|p_{\st}(X)p_{\st}(Y)}{p_{\st}(A\times\{i\})P_{\st}(A\times\{i\},A'\times\{i-1\})}
&\le \fc{2(L-1) 2p_{\st}(S)p_{\st}(S^c)}{p_{\st}(A\times\{i\})P_{\st}(A\times\{i\},A'\times\{i-1\})}
\end{align}
where $S= A\times \{i,\ldots,L\}$ is the union of all children of $A\times \{i\}$ (including itself). This follows because the paths which go through $(A\times\{i\},A'\times\{i-1\})$ are exactly those between $X,Y$ where one of $X,Y$ is a subset of $S= A\times \{i,\ldots, L\}$ and the other is not. To upper bound~\eqref{eq:st-path}, we upper-bound $\fc{p(S)}{p(A\times \{i\})}$ and lower-bound $P(A\times \{i\}, A'\times \{i+1\})$. 

We upper-bound by definition of $\ga$,
\begin{align}
\fc{p(S)}{p(A\times \{i\})} &= \fc{\sum_{k=i}^{L} r_ip_i(A)}{r_ip_i(A)}\\
&\le\fc{\max r_i}{\min r_i}
 \fc{\sum_{k=i}^{L} p_i(A)}{p_i(A)}\\
&\le \fc L{r\ga}.
\label{eq:path-bd1}
\end{align}

%

Next we lower bound $P_{\st}(A\times \{i\}, A'\times \{i-1\})$.  There is probability $\rc{2L}$ of proposing a switch to level $i-1$, so
\begin{align}
P_{\st}(A\times \{i\}, A'\times \{i-1\})
&\ge \rc{2L} \int_\Om r_ip_i(x) \min\bc{
\fc{p_{i-1}(x)}{p_i(x)}\fc{r_{i-1}}{r_i}, 1
}\dx/(r_ip_i(A))\\
&=\rc{2L}\int_\Om  \min\bc{p_{i-1}(x)\fc{r_{i-1}}{r_i}, p_i(x)}\dx/p_i(A)\\
&\ge \rc{2L} \fc{\min r_j}{\max r_j} \int_\Om \min\bc{p_{i-1}(x), p_i(x)} \dx /  p_i(A)\\
&\ge \rc{2L} r\de.
\label{eq:path-bd2}
\end{align}
Putting~\eqref{eq:st-path},~\eqref{eq:path-bd1}, and~\eqref{eq:path-bd2} together, 
\begin{align}
\eqref{eq:st-path}&\le 2(L-1)2\pf{L}{r\ga} \pf{2L}{r\de}\\
&\le \fc{8L^3}{r^2\ga \de}.
\end{align}
Using~\eqref{eq:st-gap-prod} and Theorem~\ref{thm:can-path}, 
\begin{align}
\Gap(M_{\st}) &\ge \rc 2 \Gap(\ol M_{\st})\min_{B\in \cal P} \Gap(M_{\st}|_B)\\
&\ge \fc{r^2\ga \de}{16L^3} \min_{B\in \cal P} \Gap(M_{\st}|_B)\\
&\ge \fc{r^2\ga \de}{32L^3} \min_{1\le i\le L, A\in \cal P_i} \Gap(M_{i}|_A)
\end{align}
\end{proof}

By taking all the partitions except the first to be the same, we see that this theorem is an improvement to the bound for simulated tempering in~\cite[Theorem 3.1]{woodard2009conditions}, which gives the bound 
$$
\Gap(P_{\st}) \ge\fc{\ga^{J+3} \de^3}{2^{14}(L+1)^5J^3} \min\bc{
\min_{2\le i\le L, A\in \cal P} (\Gap(M_{i}|_A)), \Gap(M_1)}
$$
when $r=1$, where $J$ is the number of sets in the partition. Most notably, their bound is exponential in $J$, while our bound has no dependence on $J$.
%